\documentclass{article}

\PassOptionsToPackage{numbers,sort&compress}{natbib}
\usepackage[preprint,main]{neurips_2026}

\usepackage[T1]{fontenc}
\usepackage{microtype}
\usepackage{graphicx}
\usepackage{subcaption}
\usepackage{booktabs}
\usepackage{multirow}
\usepackage{amsmath, amssymb, amsfonts}
\usepackage{amsthm}
\usepackage{mathtools}
\usepackage{algorithm}
\usepackage{algorithmic}
\usepackage{hyperref}
\usepackage{enumitem}
\usepackage{xcolor}
\newcommand{\dop}{\mathrm{do}}

\usepackage[capitalise,noabbrev]{cleveref}

\newtheorem{remark}{Remark}

\newtheorem{definition}{Definition}
\newtheorem{proposition}{Proposition}
\newtheorem{lemma}{Lemma}
\newtheorem{theorem}{Theorem}
\newtheorem{corollary}{Corollary}
\newtheorem{assumption}{Assumption}

\newcommand{\E}{\mathbb{E}}
\newcommand{\Var}{\mathrm{Var}}

\DeclareMathOperator*{\argmin}{arg\,min}
\DeclareMathOperator*{\argmax}{arg\,max}

\title{Causal Mechanism Reduction:\\
Mechanism Replacement for Neural Network\\
Pruning and Abstraction}

\author{%
  Amir Asiaee \\
  Vanderbilt University Medical Center \\
  Nashville, TN \\
  \texttt{amir.asiaeetaheri@vumc.org}
}

\begin{document}
\maketitle

\begin{abstract}
Which internal mechanisms of a neural network can be replaced while
preserving the computation it performs? This question has two
versions: structured pruning asks for smaller deployable networks, while
causal abstraction asks for high-level models that commute with
interventions. We introduce \emph{causal mechanism reduction} (\textbf{CMR}), a
framework for mechanism replacement that treats a trained network as a
deterministic structural causal model and replaces selected internal
variables by constants or affine functions of retained variables. These
replacements compile exactly into smaller dense networks through bias and
weight folding, and they induce reduced causal models that can be tested
with interchange interventions.

We derive a unified second-order replacement-risk objective with special
cases that recover mean replacement, variance-based pruning, logit-distortion
scoring, and affine neuron merging. A margin-based certificate connects
logit distortion to interchange-intervention agreement, giving a formal
link between compression error and abstraction fidelity. The framework
also identifies a basic invariance requirement for mechanism scores:
functionally identical ReLU networks should induce the same reduction. In
exact positive-scaling reparameterizations, variance-based pruning
violates this requirement, with kept-set Jaccard collapsing to the chance level for the budget (about one third), while our
logit-distortion score is exactly invariant.

In experiments, CMR variants are competitive with VBP under matched
fine-tuning of DeiT-Tiny on ImageNet-100, where final-block FFN pruning is
largely recovered by fine-tuning. The clearer separation appears in the
invariance and interchange tests: CMR-Logit preserves kept sets under
reparameterization and consistently improves distributional fidelity under
interchange interventions, with small, directionally consistent
interchange-accuracy gains. CMR thus provides a common object for pruning,
compilation, and causal-abstraction verification.
\end{abstract}

\section{Introduction}
\label{sec:intro}

Neural networks are often simplified after training. In compression, the
simplified model should be smaller and faster without losing predictive
behavior. In causal abstraction and mechanistic interpretability, the
simplified model should preserve how internal variables support behavior
under interventions. These goals are usually studied separately, but
they share a basic question: which internal mechanisms
can be replaced without changing what the network does?

The common operation is \emph{mechanism replacement}. A pruning method may
replace a hidden unit by a constant and fold the resulting contribution
into the next bias; an abstraction method may keep a subset of internal
variables and ask whether interventions on those variables commute with
the original computation. In this view, structured pruning and causal
abstraction differ less in the operation they perform than in the
fidelity metric they demand: deployable task behavior for pruning,
interchange-intervention agreement for abstraction.

We formalize this operation as \emph{causal mechanism
reduction}. Given a trained feedforward network viewed as a deterministic
structural causal model (\textbf{SCM}) \citep{pearl2009causality}, we replace
selected units by constants or by affine functions of retained units. The
same reduced object has two readings. Computationally, it compiles exactly
into a smaller dense network through bias and weight folding. Causally,
it defines an explicit reduced SCM whose state map is the projection onto
retained activations and whose fidelity can be checked under interchange
interventions \citep{geiger2021causal,beckers2019abstracting}.

This is not only a change of vocabulary. A unified
second-order replacement-risk theorem (\Cref{thm:cmr-risk}) shows that
mean replacement, variance-based pruning (\textbf{VBP}) \citep{berisha2025vbp},
logit-distortion scoring, and affine neuron merging
\citep{DBLP:conf/nips/KimKPJ20} are all special cases of the same local
quadratic objective. A margin-based interchange-fidelity certificate
(\Cref{thm:margin}) then connects one tractable surrogate, expected
squared logit distortion, to class-level agreement under interchange
interventions. The same replacement score can therefore be read both as a
compression criterion and as a proposal for a verifiable causal
abstraction.

The formalism also reveals a concrete failure mode of an existing pruning
heuristic. ReLU networks admit exact function-preserving rescalings:
multiplying a hidden unit by a positive scalar and dividing its outgoing
weights by the same scalar leaves the network function unchanged. A
behavioral reduction criterion should not identify different mechanisms
in two such networks. VBP is not invariant to this transformation because
activation variance changes with scale; logit distortion is invariant. In
our stress test, exact ReLU scaling reparameterizations on $[0.01, 100]$
drive VBP's kept-set Jaccard to chance-level overlap ($0.35$
against a random-subset floor of $\approx 1/3$) across functionally
identical networks, while CMR-Logit remains at Jaccard $1.0$, with
downstream effects on interchange-intervention agreement and
distributional fidelity.

Empirically, the results are deliberately two-sided. On
DeiT-Tiny final-block FFN pruning after transfer to ImageNet-100, all
methods land within $0.7$ top-1 points of the $84.62$ baseline after
matched fine-tuning (most cells slightly above it), including random pruning; this indicates that the
CMR variants are competitive with VBP in a modern transformer setting,
but also shows that top-1 accuracy is saturated in this protocol. The
clearest method signal appears in the invariance stress test, while
interchange-intervention experiments on CIFAR-10 ConvNet and ResNet-56 /
CIFAR-100 verify that compiled reductions can be evaluated as approximate
causal abstractions, not only as compressed predictors.

\paragraph{Contributions.}
\begin{itemize}[leftmargin=1.25em]
\item \textbf{A unified replacement primitive.} We define constant and
affine mechanism replacements that compile into smaller dense
networks and induce explicit reduced SCMs that can be verified under interchange
interventions.
\item \textbf{Theory connecting pruning scores to abstraction fidelity.}
We prove a unified replacement-risk theorem that recovers VBP, Logit-MSE /
CMR-Logit, and affine weighted-least-squares (\textbf{WLS}) scoring as special cases, and a margin-based
certificate connecting logit distortion to interchange-intervention
agreement.
\item \textbf{Evidence for when the abstraction lens matters.} We show a
large reparameterization-invariance failure for VBP, verify compiled
reductions under interchange interventions, and compare CMR variants with
VBP, magnitude, random, and DepGraph baselines across DeiT-Tiny /
ImageNet-100, CIFAR-10 ConvNet, and ResNet-56 / CIFAR-100.
\end{itemize}

\paragraph{Notation and roadmap.}
\Cref{sec:related} places the work relative to causal abstraction,
structured pruning, and second-order compression.
\Cref{sec:setup} defines mechanism replacement for deterministic neural
SCMs, including interventional risk and exact compilation.
\Cref{sec:cmr-theorem} presents the replacement-risk theorem, the
interchange-fidelity certificate, and the invariance result.
\Cref{sec:experiments} reports the invariance stress test, modern
transformer benchmark, and interchange-fidelity study. \Cref{sec:limitations}
discusses limitations.

\section{Related Work and Positioning}
\label{sec:related}

\paragraph{Causal abstraction and mechanistic interpretability.}
Causal abstraction asks when a high-level structural causal model (SCM)
commutes with a low-level one under interventions
\citep{rubenstein2017causal,beckers2019abstracting}; approximate
abstraction allows graded mismatch between low- and high-level
models \citep{beckers2020approximate}. Interchange interventions
turn this criterion into a neural-network test by swapping internal states
and measuring interchange-intervention accuracy (\textbf{IIA}) \citep{geiger2021causal}.
Interchange-intervention training shows that behavioral accuracy can miss
mechanistic failures \citep{geiger2022iit}. Later work studies
distributed alignment search and frames causal abstraction as a formal
language for mechanistic interpretability
\citep{geiger2024finding,geiger2024causal}. Here, abstraction is
constructive: we search over restricted mechanism replacements that produce
an explicit reduced SCM and verify the compiled model under interchange
interventions.

\paragraph{Alignment-map complexity.}
Recent critiques sharpen the role of the abstraction map. With sufficiently
expressive nonlinear alignment maps, high IIA can be uninformative even for
randomly initialized networks \citep{sutter2025nonlinear}. CMR avoids this
failure mode by fixing the maps: the state map is the fixed projection onto
retained activations, the intervention map is the identity on retained
coordinates, and the reported IIA is tied to the margin certificate of
\Cref{thm:margin} rather than to a learned alignment.

\paragraph{Structured and activation-statistic pruning.}
Structured pruning removes units, channels, heads, or blocks, yielding
smaller dense computations rather than
sparse masks. Classical compression pipelines combine pruning with
quantization and coding \citep{han2016deepcompression}; structured CNN
methods remove filters or channels using magnitude, sparsity-inducing gates,
or Taylor criteria
\citep{li2017pruningfilters,liu2017networkslimming,he2017channelpruning,molchanov2017pruning}.
Variance-based pruning (VBP) removes low-variance MLP units and folds mean
activations into downstream biases \citep{berisha2025vbp}; mean replacement,
neuron merging, dependency-aware pruning, and submodular structured pruning
use related selection and replacement principles
\citep{evci2018mean,DBLP:conf/nips/KimKPJ20,fang2023depgraph,elhalabi2022data}.
CMR treats these operations as instances of mechanism replacement:
constant replacement recovers mean folding and VBP under stationarity and
uniform curvature, while affine replacement connects to neuron merging and
soft interventions in causal models \citep{massidda2023causal}. Beyond a new score, this yields a shared objective that exposes when an
activation statistic is invariant, when it is fragile, and when the reduced
network should be read as an approximate abstraction.

\paragraph{Second-order and modern transformer pruning.}
Classical second-order pruning methods such as Optimal Brain Damage and
Optimal Brain Surgeon use Taylor expansions in weight space
\citep{DBLP:conf/nips/CunDS89,DBLP:conf/nips/HassibiS92}. Modern variants
scale this idea by approximating or inverting curvature for large networks, including
WoodFisher, OBC/GPTQ/SparseGPT, and Hessian-diagonal tooling
\citep{singh2020woodfisher,frantar2022optimal,frantar2023sparsegpt,dangel2020backpack,elsayed2024revisiting}.
Recent transformer and LLM pruning methods such as Wanda, LLM-Pruner, fast
BERT pruning, and X-Pruner provide strong practical baselines for weight,
channel, or block removal \citep{sun2024wanda,ma2023llmpruner,kwon2022fast,yu2023xpruner};
attention-head pruning shows that many transformer heads can be removed with
small behavioral loss \citep{michel2019sixteen,voita2019analyzing}.
CMR targets a different object: it scores activation mechanisms and their
replacements, compiles the selected interventions by bias or weight folding,
and gives the resulting dense network an explicit interventional semantics.


\section{Problem Setup: Mechanism Replacement}
\label{sec:setup}

\subsection{Networks as deterministic SCMs}
\label{sec:scm}

Let $f_\theta:\mathcal{X}\to\mathbb{R}^q$ be a trained feedforward
network with fixed parameters $\theta$. Given a task loss
$\ell:\mathbb{R}^q\times\mathcal{Y}\to\mathbb{R}_{\geq 0}$ and a
calibration set
$\mathcal{D}_{\mathrm{cal}}=\{(x_s,y_s)\}_{s=1}^n$, define
$L(\theta):=n^{-1}\sum_{s=1}^n\ell(f_\theta(x_s),y_s)$.
(We overload $\ell$ as the task loss and as a layer index; the
loss always carries arguments, and layer indices appear as sub- or
superscripts.)
For layer $\ell \in [L]$ with width $d_\ell$, let
$a^{(\ell)}(x)\in\mathbb{R}^{d_\ell}$ denote the post-nonlinearity
activation vector and let $a^{(0)}(x):=x$. The calibration activations are
$A^{(\ell)}\in\mathbb{R}^{n\times d_\ell}$ with
$A^{(\ell)}_{s,j}:=a^{(\ell)}_j(x_s)$.
For a unit $j$, write $\mathbf{a}_j:=A^{(\ell)}_{:,j}$, with layer
superscripts omitted when clear; let
$\bar a_j:=n^{-1}\mathbf{1}_n^\top\mathbf{a}_j$ and
$\Var[\mathbf{a}_j]:=n^{-1}\|\mathbf{a}_j-\bar a_j\mathbf{1}_n\|_2^2$.

A feedforward network is also viewed here as a deterministic SCM over its internal
activations. The exogenous input is $X$, the endogenous variables are the
activations $\{a_j^{(\ell)}\}_{\ell,j}$, and the structural equations are
the forward computations
$z^{(\ell)}=W^{(\ell)}a^{(\ell-1)}+b^{(\ell)}$ and
$a^{(\ell)}=\sigma^{(\ell)}(z^{(\ell)})$.
This SCM view is not a claim that the network recovers exogenous
real-world causes. It supplies intervention semantics for a deterministic
computation graph.

\subsection{Replacement, risk, and compilation}
\label{sec:compile}

CMR modifies this SCM by replacing selected internal mechanisms with the
trained weights held fixed. For a unit $j$ in layer $\ell$, the
constant replacement $\dop(a_j^{(\ell)}:=c)$
severs the incoming edges to that unit and sends the constant $c$ to all
downstream consumers. More generally, fix a retained set
$K\subseteq[d_\ell]$ and a replaced set $S=[d_\ell]\setminus K$. A
replacement class $\Phi$ maps retained activations to replacements for
$A_S$: constants use $\phi(A_{s,K})=c$, while affine replacements use
$\phi(A_{s,K})=\beta+B A_{s,K}$ or a restricted parent subset.

For a single constant intervention, let
$f_\theta^{\dop(\ell,j:=c)}$ be the intervened network and define
$L_{\ell,j}(c):=n^{-1}\sum_{s=1}^n
\ell(f_\theta^{\dop(\ell,j:=c)}(x_s),y_s)$.
The corresponding effect is
$\Delta L_{\ell,j}(c):=L_{\ell,j}(c)-L(\theta)$. Structured compression
can therefore be posed as a constrained intervention-selection problem:
\begin{equation}
\label{eq:compression-objective}
\min_{\substack{\mathcal{I} \subseteq [L] \times [d_\bullet],\;
|\mathcal{I}| = k \\ \{c_{(\ell,j)}\}_{(\ell,j) \in \mathcal{I}}}}
\quad L^{\dop(\mathcal{I}, \mathbf{c})}.
\end{equation}
The exact objective is expensive to search directly, so \Cref{sec:cmr-theorem}
derives the local replacement-risk proxy used for scoring.

\begin{proposition}[Bias-folding equivalence]
\label{prop:bias-fold}
Consider layer $\ell$ with output $a^{(\ell)} \in \mathbb{R}^{d_\ell}$,
followed by an affine transformation $u=W a^{(\ell)}+b$,
where $W \in \mathbb{R}^{m \times d_\ell}$ and $b \in \mathbb{R}^m$.
Suppose unit $j$ is clamped to constant $c$. Define
$W':=W_{:,\backslash j}\in\mathbb{R}^{m\times(d_\ell-1)}$ and
$b':=b+cW_{:,j}$. Then, for all $a^{(\ell)}_{\backslash j}$,
$W a^{(\ell)}+b|_{a^{(\ell)}_j=c}=W'a^{(\ell)}_{\backslash j}+b'$.
\end{proposition}

For a set $S$ of constant replacements, the same folding accumulates as
$b'=b+\sum_{j\in S}c_jW_{:,j}$ and $W'=W_{:,\backslash S}$.
Affine replacements fold analogously by redistributing the replaced unit's
outgoing column onto the parent columns and bias. Thus CMR can be read in two ways:
as a logical intervention in the neural SCM and as a smaller
dense network with no runtime masking (\Cref{fig:overview}). All proofs and longer derivations for
the main-text statements are collected in Appendix~\ref{app:deferred}.

\section{Theory: Replacement Risk, Fidelity, and Invariance}
\label{sec:cmr-theorem}

\subsection{Unified replacement risk}
\label{sec:taylor}

The central score is a local approximation to the loss incurred by replacing
mechanisms at a fixed layer. For a replaced set $S$ and retained set
$K=[d_\ell]\setminus S$, write the replacement perturbation on sample $s$ as
$\delta_s:=\phi(A_{s,K})-A_{s,S}$.

\begin{theorem}[Unified replacement-risk decomposition]
\label{thm:cmr-risk}
Fix a layer $\ell$, a retained set $K \subseteq [d_\ell]$, a replaced
set $S = [d_\ell] \setminus K$, a replacement class $\Phi$ (constants,
affine functions of $A_K$, or learned mechanisms), and a discrepancy
$d$ on logits or task loss that is twice differentiable at the observed
activations. The second-order local replacement-risk proxy has the
quadratic form
\[
Q_S(\phi) \;=\; \E_s\!\Bigl[\, g_{s,S}^{\top}\delta_s
+ \tfrac12\,\delta_s^{\top} H_{s,S}\,\delta_s \,\Bigr],
\qquad
\delta_s \;:=\; \phi(A_{s,K}) - A_{s,S}.
\]
Under block-diagonal or diagonal curvature, $Q_S$ decomposes across
replaced mechanisms, yielding independent per-unit or per-group scores.
The constant, affine, VBP \citep{berisha2025vbp}, logit-distortion, and
neuron-merging \citep{DBLP:conf/nips/KimKPJ20} cases follow as named
special cases.
\end{theorem}

For single-unit constant replacement $A^{(\ell)}_{s,j}\mapsto c$, the
proxy is minimized by
\begin{equation}
\label{eq:cstar-summary}
c^\star_{\ell,j}
=
\frac{\mathbf{h}^\top \mathbf{a}_j-\mathbf{1}_n^\top\mathbf{g}}
{\mathbf{1}_n^\top\mathbf{h}},
\end{equation}
provided $\mathbf{1}_n^\top\mathbf{h}>0$. Under samplewise gradient stationarity ($g_s=0$ for all $s$)
and uniform positive curvature, this reduces to mean replacement and
ranking by activation variance, recovering VBP. Under squared logit
distortion with downstream logits affine in unit $j$, the same theorem
gives $c^\star=\bar a_j$ and
$s^{\mathrm{logit}}_{\ell,j}
=\Var[\mathbf{a}_j]\,\|W_{:,j}\|_2^2$.
For affine replacement, the minimizer is a curvature-weighted least-squares
fit; for block-diagonal curvature, group scores add and a size-$k$
replacement set is selected by the bottom-$k$ scores. \Cref{tab:special-cases}
summarizes these cases and the corresponding methods.

The assumptions are local: the network is fixed, the
computation graph is deterministic and feedforward, replacements are
constants or affine functions of retained variables, downstream consumers
are affine when we claim exact folding, and diagonal or block-diagonal
curvature is used for scoring rather than for compilation or verification.

\subsection{Interchange-fidelity certificate}
\label{sec:margin}

Replacement risk is a cheap candidate score; causal-abstraction
verification asks whether the reduced model commutes with the original under
interchange interventions \citep{geiger2021causal,beckers2019abstracting}.
For an interchange intervention $I$ on retained coordinates, the state map
$\tau$ is projection onto retained activations and the intervention map
$\omega$ applies the corresponding low-level intervention. The following
certificate links expected logit distortion to IIA when the low-level
intervened model has margin.

\begin{theorem}[Margin-based interchange-fidelity certificate]
\label{thm:margin}
Let $z_L^I(x)$ denote the low-level network's logit vector under
interchange intervention $I$, and let $z_H^{\omega(I)}(x)$ denote the
compiled high-level model's logits under the corresponding intervention
$\omega(I)$. Define the interchange margin as
$m_I(x) := z_{L,y(x)}^I(x) - \max_{y' \neq y(x)} z_{L,y'}^I(x)$, where
$y(x) := \argmax_y z_{L,y}^I(x)$. Then for every $\epsilon > 0$,
\[
\Pr\bigl[\hat y_H^{\omega(I)}(x) \neq \hat y_L^I(x)\bigr]
\;\leq\;
\Pr\bigl[m_I(x) \le 2\epsilon\bigr]
\;+\;
\Pr\bigl[\|z_H^{\omega(I)}(x) - z_L^I(x)\|_\infty > \epsilon\bigr].
\]
By Markov's inequality, the second term is bounded by
$\E\bigl[\|z_H^{\omega(I)}(x) - z_L^I(x)\|_\infty^2\bigr] / \epsilon^2$,
which is upper-bounded by the expected squared logit distortion $D_2$ of the compiled model under the verification distribution.
\end{theorem}
The CMR-Logit score estimates $D_2$ when the calibration
distribution matches the interchange marginal on the replaced units and
cross-unit distortion terms are accounted for; the experiments evaluate
the certificate with the empirically measured joint $D_2$.

Equivalently, if $D_2$ is the unnormalized expected squared logit
distortion and $M(t):=\Pr[m_I(x)\le t]$, then for every $\epsilon>0$,
$\mathrm{IIA}\ge 1-M(2\epsilon)-D_2/\epsilon^2$.
When $m_I(x)\ge\gamma$ almost surely, this gives
$\mathrm{IIA}\ge 1-4D_2/\gamma^2$. The margin term is the unavoidable
failure mode: near decision boundaries, no small logit-distortion guarantee
can certify class agreement.

\subsection{Reparameterization invariance}

ReLU networks admit exact function-preserving symmetries: for any $s>0$,
multiplying a hidden unit's activations by $s$ and dividing its outgoing
weights by $s$ leaves the network function unchanged. Variance is not
invariant to this transformation; expected squared logit distortion is invariant.

\begin{proposition}[ReLU scaling invariance of CMR-Logit]
\label{prop:relu-scaling-invariance}
Let $f_\theta$ be a feedforward ReLU network, and let $j$ be a hidden
unit in layer $\ell$ with post-ReLU activation $a_j(x)$ and outgoing
weight column $W_{\ell+1,:,j}$. For any $s>0$, scale the incoming
weights and bias of unit $j$ by $s$, so that its activation becomes
$a'_j(x)=s a_j(x)$ for all $x$, and scale its outgoing column by
$W'_{\ell+1,:,j}=s^{-1}W_{\ell+1,:,j}$. Then the network function is
unchanged for every input $x$. Under this transformation,
\[
\Var[a'_j]\cdot\|W'_{\ell+1,:,j}\|_2^2
=
\Var[a_j]\cdot\|W_{\ell+1,:,j}\|_2^2,
\]
so the CMR-Logit ranking is preserved exactly, up to global normalization
constants. By contrast, $\Var[a'_j]=s^2\Var[a_j]$; with independent
positive scalings across units, variance-based rankings can be changed
arbitrarily among units with nonzero variance.
\end{proposition}

\section{Experiments}
\label{sec:experiments}

We organize the empirical evaluation around three claims: (1)
coordinate-invariance separates
methods on functionally identical networks (\Cref{sec:exp-invariance});
(2) CMR is competitive on a modern transformer pruning benchmark,
where matched fine-tuning makes top-1 accuracy a saturated recovery metric
(\Cref{sec:exp-benchmark}); and (3) the compiled reductions behave as
approximate abstractions under interchange interventions
(\Cref{sec:exp-verification}). We close with diagnostics that show when
local scores are reliable and when retraining or recomputation dominates
the ranking signal.

\paragraph{Methods.}
We evaluate three variants of CMR (CMR-Const, CMR-Logit, CMR-Affine)
and compare them with variance-based pruning \citep{berisha2025vbp}, magnitude pruning,
and the structural-pruning baseline DepGraph \citep{fang2023depgraph}. A
random selector serves as an unstructured recovery reference.
\Cref{tab:special-cases} maps the methods and related baselines to
special cases of \Cref{thm:cmr-risk}.

\paragraph{Compute.}
DeiT-Tiny experiments use a single NVIDIA RTX 4500 Ada (24 GB VRAM,
CUDA 13.0); all reported throughput numbers are measured on this device.
ConvNet and ResNet experiments use cached checkpoints and run on CPU,
which is sufficient for these smaller workloads and ensures the
interchange-intervention budget  is held
constant across methods within each experiment ($R = 2000$ swaps
per cell for verification and the off-diagonal diagnostic, $R = 1000$
for the affine ablation, $R = 500$ for the calibration-shift study).

\paragraph{Metrics.}
Compression experiments report retained fraction $\rho$, top-1/top-5
accuracy after any matched fine-tuning, parameters, multiply-accumulate
operations (\textbf{MACs}), and measured throughput. Abstraction
experiments report interchange-intervention accuracy (IIA;
higher is better) and KL divergence between the original and reduced
intervention logits (lower is better). The invariance stress test reports
the Jaccard overlap between the kept sets selected before and after an
exact function-preserving reparameterization.

\begin{figure}[t]
\centering
\includegraphics[width=\linewidth]{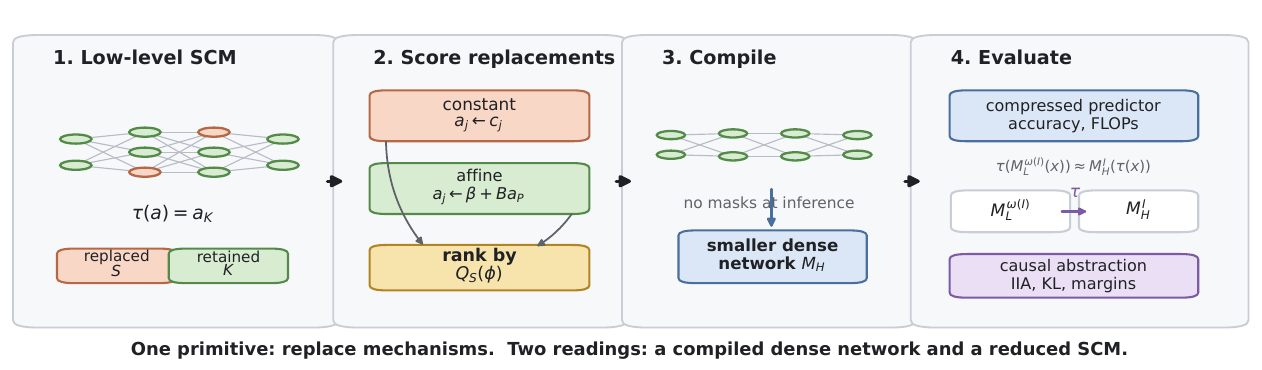}
\caption{\textbf{Mechanism replacement overview.} A low-level network is
viewed as a deterministic SCM over activations. A retained set $K$ defines
the projection $\tau$, while replaced units $S$ are assigned constant or
affine mechanisms that fold into a smaller dense network. We evaluate both
compiled task behavior and commutativity under interchange interventions.}
\label{fig:overview}
\end{figure}

\begin{table}[t]
\centering
\caption{Known structured-pruning and abstraction-discovery methods as
special cases of the unified replacement-risk objective
(\Cref{thm:cmr-risk}). Replacement objective $\times$ replacement class
$\times$ score formula $\times$ method recovered.}
\label{tab:special-cases}
\footnotesize
\setlength{\tabcolsep}{3pt}
\begin{tabular}{@{}llll@{}}
\toprule
Replacement objective & Replacement class & Score formula & Method recovered \\ \midrule
uniform curvature, samplewise stationarity & constant & $\mathrm{Var}(a_j)$ & mean replacement \citep{evci2018mean} / VBP \citep{berisha2025vbp} \\
squared logit distortion & constant & $\mathrm{Var}(a_j)\,\|W_{:,j}\|_2^2$ & Logit-MSE / CMR-Logit \\
supervised CE local risk & constant & curvature-weighted score & CMR-Const \\
supervised / logit local risk & affine & weighted residual & CMR-Affine, neuron merging \citep{DBLP:conf/nips/KimKPJ20} \\
next-layer reconstruction & affine / full set & reconstruction residual & submodular pruning \citep{elhalabi2022data} \\
\bottomrule
\end{tabular}
\end{table}
\subsection{Reparameterization-invariance stress test}
\label{sec:exp-invariance}

\Cref{prop:relu-scaling-invariance} predicts that CMR-Logit rankings are
unchanged under exact ReLU positive scaling, while VBP rankings depend on
the absolute activation scale. This is the controlled setting in which method
choice matters most: the paired networks compute the same
function, so a changed kept set reflects coordinate dependence rather than
behavioral structure. The two cached models are a small
convolutional CIFAR-10 classifier with a $256$-unit penultimate
representation and a ResNet-56 trained on CIFAR-100 whose post-GAP
representation has $64$ channels; both are scored at the
penultimate-representation--head interface. For each seed, unit-wise scalings
$s_j \sim \mathrm{LogUniform}(s_{\min}, s_{\max})$ are applied as
$a_j \mapsto s_j a_j$, $W_{:,j} \mapsto W_{:,j}/s_j$ at the
representation--head interface, i.e., the transformation that an exact
ReLU rescaling of unit $j$ induces on the cached activations and head
weights; the maximum logit
difference between original and rescaled models is below $10^{-5}$ in
every cell, a numerical sanity check on the identity.
Each cell keeps half the units ($128$ of $256$ on the ConvNet, $32$ of
$64$ on ResNet-56), with ten scaling draws over five (ConvNet) and
three (ResNet-56) trained checkpoints.

On the CIFAR-10 ConvNet \citep{krizhevsky2009learning} at $[0.01, 100]$, CMR-Logit has kept-set Jaccard
$1.000 \pm 0.000$ across ten seeds while VBP has Jaccard
$0.346 \pm 0.021$ (\Cref{tab:e2-invariance-summary}; \Cref{fig:e2-invariance}),
statistically indistinguishable from the chance floor: the expected
Jaccard of two independent random half-subsets is $\approx 1/3$, and the
measured random selector gives $0.333 \pm 0.030$. At the narrower
$[0.1, 10]$, CMR-Logit remains exactly invariant and VBP remains
scale-dependent. The same pattern appears on ResNet-56 \citep{he2016deep} /
CIFAR-100, where CMR-Logit is again exactly invariant and VBP has
Jaccard $0.374 \pm 0.037$ at $[0.01, 100]$ (random floor
$0.356 \pm 0.092$), so under the strongest scaling range VBP's
kept-set overlap falls to chance level. This is not merely label instability:
at keep $=128$ on the CIFAR-10 ConvNet under the strongest
scaling range, the downstream CMR-Logit IIA gap over VBP is $+0.131$ and
the KL gap is $-0.362$.

\begin{table}[t]
\centering
\caption{\textbf{Kept-set stability under exact ReLU reparameterization.}
Jaccard overlap is measured between the kept set selected on the original
network and the kept set selected on a functionally identical
positively-scaled network (mean $\pm$ sd over ten seeds); bold marks
the best Jaccard in each row. Each kept set retains half the
units; the chance floor for two independent random half-subsets is
$\approx 1/3$ (measured random selector: $0.333 \pm 0.030$ ConvNet,
$0.356 \pm 0.092$ ResNet-56).}
\label{tab:e2-invariance-summary}
\small
\begin{tabular}{@{}llccc@{}}
\toprule
Model & Scale range & CMR-Logit & VBP & magnitude \\ \midrule
CIFAR-10 ConvNet & $[0.01,100]$ & $\mathbf{1.000 \pm 0.000}$ & $0.346 \pm 0.021$ & $0.359 \pm 0.027$ \\
CIFAR-10 ConvNet & $[0.1,10]$ & $\mathbf{1.000 \pm 0.000}$ & $0.371 \pm 0.024$ & $0.378 \pm 0.027$ \\
ResNet-56 / CIFAR-100 & $[0.01,100]$ & $\mathbf{1.000 \pm 0.000}$ & $0.374 \pm 0.037$ & $0.352 \pm 0.049$ \\
ResNet-56 / CIFAR-100 & $[0.1,10]$ & $\mathbf{1.000 \pm 0.000}$ & $0.389 \pm 0.038$ & $0.361 \pm 0.052$ \\
\bottomrule
\end{tabular}
\end{table}

\begin{figure}[t]
\centering
\includegraphics[width=.85\linewidth]{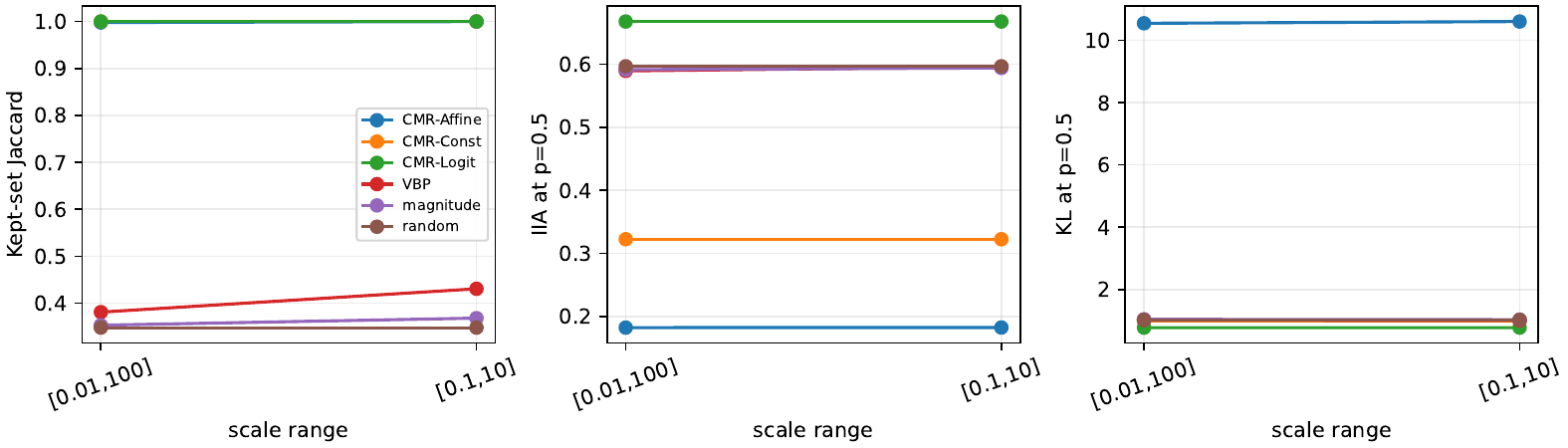}
\caption{\textbf{Reparameterization stress test.} CMR-Logit is invariant
under exact positive scaling, while VBP and magnitude rankings change
substantially across functionally identical networks. Curves pool
three model populations (CIFAR-10 ConvNet, ResNet-56 / CIFAR-100, and an
untrained ResNet-20 control), so plotted values differ from the
per-model rows of \Cref{tab:e2-invariance-summary}; the random
selector's spread reflects independent re-draws, not scale dependence.}
\label{fig:e2-invariance}
\end{figure}

\subsection{Modern pruning benchmark}
\label{sec:exp-benchmark}

We initialize DeiT-Tiny from ImageNet-1K pretrained weights
\citep{touvron2021training} and fine-tune it for five epochs on
ImageNet-100 \citep{tian2020contrastive} (the standard 100-class subset of ImageNet-1K \citep{russakovsky2015imagenet}, $\approx
127$K train / $5$K val) to a baseline of top-1 / top-5 $84.62 / 97.32$.
We then prune the final transformer block's feed-forward (\textbf{FFN})
intermediate units at keep fractions $\{0.75, 0.50, 0.25\}$ using each
scoring method and apply a matched ten-epoch fine-tune (identical
optimizer, schedule, and initialization from the same baseline
checkpoint; mini-batch order varies per cell). Scores are computed from
class-token activations of the final-block FFN on a fixed calibration
set of $n = 1024$ images, using the block's down-projection output as a
local logit surrogate for the CMR variants; the pruned projections act
on all tokens at deployment. Each cell is a single seed.
Throughput is measured in half precision over five timed sweeps of $1024$
validation images after two warmup sweeps; evaluation uses a direct
$224 \times 224$ resize, applied identically to baseline and pruned
models.

After fine-tuning, all six methods finish within $0.7$ top-1 points of the
baseline (\Cref{tab:e1-imagenet100-grid,fig:e1-pruning}; pruned models
$5.32$--$5.47$M parameters, $1.04$--$1.06$G MACs, throughput
$7.4$--$7.8$K~img/s on the RTX~4500 Ada). Before fine-tuning, the
picture is sharper: CMR-Logit, VBP, magnitude, and random preserve
baseline accuracy zero-shot (top-1 $83.2$--$84.6$), CMR-Const drops
moderately ($71.6$--$76.1$), and CMR-Affine collapses to near-chance
($2.6$--$4.4$); the matched fine-tune recovers every method, which is
precisely why the post-fine-tuning grid saturates. This makes the benchmark a
recovery and deployment-viability test rather than a statistically
reliable ranking test. Among non-random selectors, CMR-Logit is highest
at keep $=0.75$ ($85.16$ vs.\ VBP $84.50$), while CMR-Affine is highest
at keep $=0.50$ and $0.25$ ($85.04$ and $85.16$, against VBP's $85.00$
and $84.70$). CMR-Const also slightly exceeds VBP at the most aggressive
setting ($84.96$ vs.\ $84.70$).

Random pruning followed by matched
fine-tuning is competitive, including the best single cell
in the grid ($85.30$ at keep $=0.25$), suggesting that the final
DeiT-Tiny FFN has enough redundant expansion capacity for ten epochs of
fine-tuning to recover from many reasonable selections. We therefore use
ImageNet-100 to show that CMR variants remain competitive with VBP
on a modern transformer, and use the invariance stress test
(\Cref{sec:exp-invariance}) for the clear method-choice separation.

\begin{table}[t]
\centering
\caption{\textbf{ImageNet-100 top-1 after matched fine-tuning.}
DeiT-Tiny final-block FFN pruning at retained fraction $\rho$; baseline
top-1 is $84.62$; bold marks the best top-1 in each column.}
\label{tab:e1-imagenet100-grid}
\small
\begin{tabular}{@{}lccc@{}}
\toprule
Method & $\rho=0.75$ & $\rho=0.50$ & $\rho=0.25$ \\ \midrule
CMR-Logit & $\mathbf{85.16}$ & $84.78$ & $84.88$ \\
CMR-Affine & $84.70$ & $\mathbf{85.04}$ & $85.16$ \\
CMR-Const & $84.66$ & $84.66$ & $84.96$ \\
VBP & $84.50$ & $85.00$ & $84.70$ \\
magnitude & $84.70$ & $84.78$ & $84.62$ \\
random & $84.68$ & $84.82$ & $\mathbf{85.30}$ \\
\bottomrule
\end{tabular}
\end{table}

On cached ResNet-56 / CIFAR-100 without fine-tuning, CMR-Logit and VBP
prune the post-GAP representation; DepGraph physically removes
stage-three residual channels and propagates through to the classifier,
with channels ranked by classifier-input column norm and the dependency
graph used for propagation (single seed, no fine-tuning).
At keep $0.75/0.50/0.25$, CMR-Logit top-1 is $0.670/0.534/0.284$, VBP
$0.672/0.526/0.286$, DepGraph $0.253/0.065/0.020$. DepGraph gives
larger parameter and MAC reductions but at a different deployment
budget; a matched-budget fine-tuned comparison under DepGraph's
native group-norm importance is left to future work.

\subsection{Interventional self-abstraction verification}
\label{sec:exp-verification}

We verify compiled reductions under $R = 2000$ Bernoulli interchange
interventions at $p = 0.5$ on cached CIFAR-10 ConvNet and ResNet-56 /
CIFAR-100 checkpoints (\Cref{tab:e3-verification-summary,fig:e3-verification};
five seeds per cell, over five ConvNet and three ResNet-56 trained
checkpoints).
On the CIFAR-10 ConvNet, CMR-Logit lies on the interchange-fidelity /
compression frontier with IIA/KL $0.858/0.082$ at keep $0.75$,
$0.675/0.431$ at keep $0.50$, and $0.483/1.034$ at keep $0.25$. VBP is
close behind on cell means at the first two keep fractions,
though the paired per-seed IIA deltas are small ($+0.007 \pm 0.016$ at
keep $0.75$, $+0.009 \pm 0.017$ at keep $0.50$) and not individually
significant, and
ResNet-56 / CIFAR-100 shows the same small CMR-Logit edge at keep $0.75$
and $0.50$. A powered MNIST \citep{lecun1998mnist} baseline with ten-seed paired bootstrap CIs
(Appendix~\ref{app:mnist-baseline}) matches this pattern
qualitatively: the IIA edge is small and not significant, while the KL
edge is significant; a
Boolean-circuit sanity check (Appendix~\ref{app:boolean}) confirms
CMR-Logit recovers the ground-truth compositional structure.

CMR-Const shows a useful failure mode for local risk scores. When each
replaced unit's optimal constant is fitted independently, simultaneously
folding the constants ignores cross-unit curvature interactions. Jointly
refitting the constants over the replaced block at compile time mitigates
the worst interaction; the corrected CMR-Const IIA on the CIFAR-10
ConvNet is $0.720/0.582/0.455$ at keep $0.75/0.50/0.25$. This is below
CMR-Logit but above the uncorrected constant-replacement failure mode,
and it matches the off-diagonal curvature diagnostic in
\Cref{sec:exp-ablations}. CMR-Affine is unstable in this
no-fine-tuning regime (ConvNet IIA $0.388/0.282/0.217$, below the random
selector's $0.723/0.538/0.378$); we quantify and discuss this in
\Cref{sec:limitations}.

\begin{table}[t]
\centering
\caption{\textbf{Interventional verification summary.} Each cell reports
IIA/KL for compiled reductions, means over five seeds; higher IIA and lower KL are better.
Bold marks per-model, per-column winners for each metric.}
\label{tab:e3-verification-summary}
\small
\begin{tabular}{@{}llccc@{}}
\toprule
Model & Method & $\rho=0.75$ & $\rho=0.50$ & $\rho=0.25$ \\ \midrule
CIFAR-10 ConvNet & CMR-Logit & $\mathbf{0.858/0.082}$ & $\mathbf{0.675/0.431}$ & $\mathbf{0.483/1.034}$ \\
CIFAR-10 ConvNet & VBP & $0.850/0.087$ & $0.666/0.447$ & $0.456/1.077$ \\
CIFAR-10 ConvNet & CMR-Const & $0.720/0.356$ & $0.582/0.730$ & $0.455/1.185$ \\
ResNet-56 / CIFAR-100 & CMR-Logit & $\mathbf{0.628/0.697}$ & $\mathbf{0.369/1.842}$ & $\mathbf{0.159}/3.232$ \\
ResNet-56 / CIFAR-100 & VBP & $0.620/0.714$ & $0.361/1.880$ & $0.152/\mathbf{3.224}$ \\
ResNet-56 / CIFAR-100 & CMR-Const & $0.538/0.944$ & $0.303/2.207$ & $0.128/3.436$ \\
\bottomrule
\end{tabular}
\end{table}

\begin{figure}[t]
\centering
\begin{subfigure}[t]{.42\linewidth}
\centering
\includegraphics[width=\linewidth]{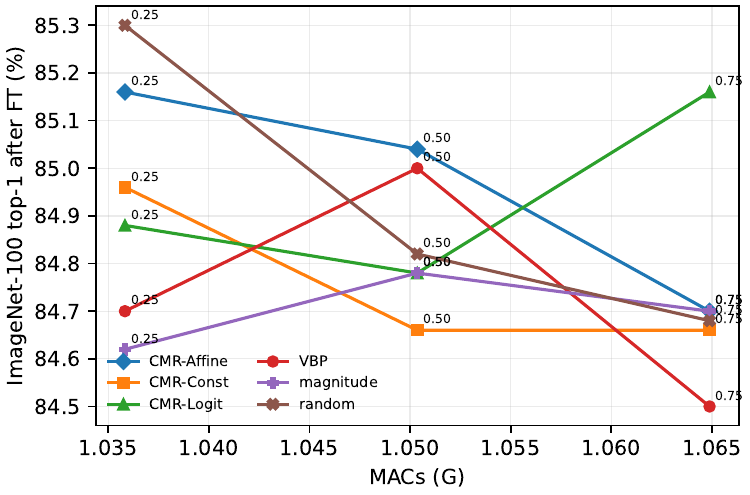}
\caption{\textbf{ImageNet-100 pruning.} Top-1 after matched fine-tuning
versus MACs; labels mark the FFN keep fraction.}
\label{fig:e1-pruning}
\end{subfigure}
\hfill
\begin{subfigure}[t]{.54\linewidth}
\centering
\includegraphics[width=\linewidth]{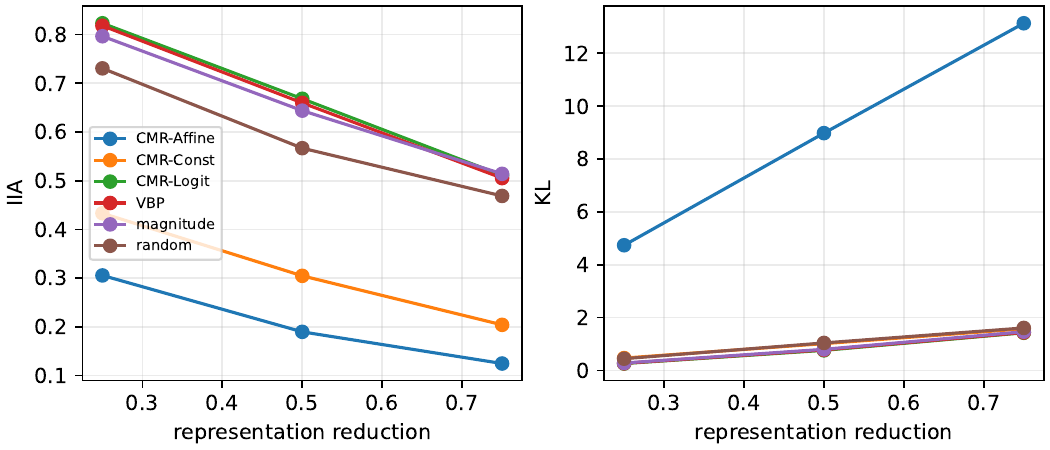}
\caption{\textbf{Interventional verification.} IIA and KL under
interchange interventions; the x-axis is representation reduction
$1-\rho$. Curves pool the CIFAR-10 ConvNet, ResNet-56 /
CIFAR-100, and an untrained ResNet-20 control; per-model values are in
\Cref{tab:e3-verification-summary}. The CMR-Affine curve shows the
no-fine-tuning instability discussed in \Cref{sec:limitations}.}
\label{fig:e3-verification}
\end{subfigure}
\caption{\textbf{Pruning and abstraction behavior.} The ImageNet-100
benchmark establishes modern-transformer recovery under matched
fine-tuning, while interchange verification tests whether compiled
reductions behave as approximate high-level causal models.}
\label{fig:experiment-panels}
\vspace{-0.75em}
\end{figure}

\subsection{Ablations and diagnostics}
\label{sec:exp-ablations}
\vspace{-0.45em}

The ablations isolate three sources of variation that could otherwise
be mistaken for method effects. First, calibration sensitivity is modest
for CMR-Logit. Under a class-subset calibration shift (calibrate on
classes $0$--$4$, evaluate on $1{,}024$ held-out CIFAR-10
test images under $R=500$ interchange interventions), CMR-Logit has
KL $0.472/0.480$ at calibration sizes $n=500/2000$, compared with
$0.616/0.599$ for VBP. CMR-Affine is more data-dependent in the cached
full-distribution sweep: at keep $128$, test accuracy rises from
$86.37$ at $n=500$ to $87.25$ at $n=10{,}000$, explaining why affine
replacement is fragile in small calibration regimes; on ImageNet-100 the compiled
affine model itself collapses zero-shot, and its recovery there is
attributable to the matched fine-tune (\Cref{sec:exp-benchmark}).
\vspace{-0.45em}

Second, affine parent selection is not the bottleneck in the small
ConvNet setting. Pearson, output-weight-aware Pearson, and random parent
sets have nearly identical mean IIA at keep $0.5$ (Appendix~\ref{app:ablations}),
while larger parent sets without sufficient regularization can worsen
KL. Third, the local one-shot approximation is stressed by
cross-unit curvature: the mean off-diagonal Hessian-mass ratio across
random 32-unit blocks of the 256-dimensional penultimate
CIFAR-10 ConvNet representation is $0.849$. In that setting, iterative
recomputation of the CMR-Const scores improves IIA from $0.584$ to $0.614$ at keep $0.5$ and
reduces KL from $0.710$ to $0.597$. Appendix~\ref{app:timing} reports
the corresponding score, compile, and verify wall-clock breakdown.
\vspace{-0.55em}

\section{Discussion and Limitations}
\label{sec:limitations}
\vspace{-0.55em}

The empirical results distinguish three questions that are often conflated
in pruning papers. ImageNet-100 asks whether mechanism replacement remains
viable in a modern transformer under matched fine-tuning; it does, but the
top-1 metric is saturated enough that method rankings are not reliable.
The reparameterization stress test asks whether a score respects the
function computed by the network rather than the coordinates used to
represent it; here the separation is sharp. Interchange verification asks
whether the compiled reduction behaves like an approximate high-level
causal model; this is where the abstraction interpretation becomes
testable rather than metaphorical.
\vspace{-0.45em}

This perspective also clarifies CMR's contribution. The primitive is not
``delete a unit'', but ``replace a mechanism and measure the behavioral
cost.'' Different choices of replacement class and discrepancy recover
variance-based pruning, logit-distortion scoring, constant replacement,
affine merging, and constructive causal abstraction as points in the same
design space. The benefit is that the assumptions behind each score become explicit and can be
tested against invariance, compilation, and interchange fidelity; no single score dominates every benchmark.
\vspace{-0.45em}

\paragraph{Locality and curvature.}
The unified replacement-risk theorem of \Cref{thm:cmr-risk} is stated for
a single layer or block. Multi-layer iterative reduction requires
recomputing activations after each replacement and re-scoring downstream
layers; our off-diagonal diagnostic shows why this matters. The diagonal
or block-diagonal curvature approximation makes scoring cheap and
additive, but cross-unit curvature matters. Iterative recomputation is the
natural extension when the one-shot approximation is too coarse, and our
diagnostic results show that it can improve interchange fidelity.
\vspace{-0.45em}

\paragraph{Replacement capacity.}
Affine replacement is more expressive than constant replacement, but the
extra capacity is not automatically useful. On small cached vision
models it can overfit the calibration signal or optimize a metric that
does not translate into low interchange KL. With matched
fine-tuning, as in the DeiT-Tiny / ImageNet-100 experiment, the same
replacement class can recover and become competitive, though the
recovery there is attributable to the fine-tune itself, which also
rescues random pruning, and the compiled affine model is near-chance
zero-shot. We therefore view
CMR-Affine as a data- and metric-dependent option rather than a uniformly
preferred selector.
\vspace{-0.45em}

\paragraph{Causal semantics.}
The causal model in this paper is a computational structural causal model
over neural activations. Its interventions are interventions on the
network computation graph, and its abstraction claim is commutativity
under a fixed state map and corresponding interchange interventions
\citep{beckers2019abstracting,geiger2021causal}. We do not claim that the
reduced model recovers exogenous real-world causal variables. It is a
reduced neural SCM whose variables may support interpretation, but whose
formal guarantee concerns computational behavior.
\vspace{-0.65em}

\section{Conclusion}
\label{sec:conclusion}
\vspace{-0.55em}

We introduced \emph{causal mechanism reduction}, a mechanism-replacement
view of neural network reduction that unifies structured pruning and
constructive causal abstraction. The unified replacement-risk theorem
recovers several existing scores as special cases, the folding results
make the reductions deployable as smaller dense networks, and the
margin-based certificate connects logit distortion to interchange
fidelity. Empirically, the main lesson is that coordinate dependence
is not a cosmetic issue: functionally identical ReLU networks can induce
different variance-based reductions, while logit-distortion scoring is
invariant and preserves better interchange behavior. Mechanism
replacement therefore offers both compression and causal-abstraction
research a common object to optimize, compile, and verify.

\bibliographystyle{plainnat}
\bibliography{main}

\clearpage \newpage
\appendix
\section{Proofs and Extended Derivations}
\label{app:deferred}

\paragraph{Roadmap.}
This appendix separates the load-bearing derivations from optional
extensions. Appendix~\ref{app:taylor} derives the single-unit quadratic
proxy, Appendix~\ref{app:special} recovers mean replacement and VBP as
special cases, Appendix~\ref{app:multi} gives the multi-unit additivity
and off-diagonal diagnostic, and Appendix~\ref{app:remainder} records
the Taylor-remainder controls. Appendix~\ref{app:cmr-proofs} then gives
the proofs for the main-text theorem, certificate, folding claim, and
invariance claim. Appendix~\ref{app:attention} is an optional grouped
mechanism extension for attention heads; it is not used by the main
experiments.

\subsection{Quadratic Proxy for Interventional Risk}
\label{app:taylor}

Exactly evaluating the interventional risk $L_{\ell,j}(c)$ requires a forward pass through the modified network for each candidate constant $c$. We now derive a computationally efficient second-order approximation that admits closed-form optimization over $c$ and yields an interpretable unit importance score.

\subsubsection{Single-Unit Constant Intervention}

Fix layer $\ell$ and unit $j$. The hard intervention $\dop(a^{(\ell)}_j := c)$ replaces the stochastic activation vector $\mathbf{a}_j \in \mathbb{R}^n$ (across calibration samples) with the constant vector $c\mathbf{1}_n$. Define the induced perturbation:
\begin{equation}
\label{eq:delta}
\boldsymbol{\delta}(c) \;:=\; c\mathbf{1}_n - \mathbf{a}_j \;\in\; \mathbb{R}^n.
\end{equation}

\paragraph{Samplewise sensitivity.}
Let $L_s := \ell(f_\theta(x_s), y_s)$ denote the loss on sample $s$. We define the per-sample gradient and curvature with respect to the scalar activation $A^{(\ell)}_{s,j}$:
\begin{equation}
\label{eq:g-def}
g_s \;:=\; \frac{\partial L_s}{\partial A^{(\ell)}_{s,j}}, \qquad
h_s \;:=\; \frac{\partial^2 L_s}{\partial (A^{(\ell)}_{s,j})^2},
\end{equation}
and collect these into vectors $\mathbf{g}, \mathbf{h} \in \mathbb{R}^n$. Since the empirical risk decomposes as a sum over independent samples, the Hessian of $L(\theta)$ with respect to the activation column $\mathbf{a}_j$ is diagonal:
\begin{equation}
\label{eq:hessian-diagonal}
\frac{\partial^2 L}{\partial \mathbf{a}_j \partial \mathbf{a}_j^\top} \;=\; \frac{1}{n}\mathrm{Diag}(\mathbf{h}).
\end{equation}

\paragraph{Second-order proxy.}
Expanding $L_{\ell,j}(c)$ to second order around the observed activations yields:
\begin{equation}
\label{eq:taylor-proxy}
\boxed{
\Delta_{\ell,j}(c)
\;:=\;
L_{\ell,j}(c) - L(\theta)
\;\approx\;
\frac{1}{n} \mathbf{g}^\top \boldsymbol{\delta}(c)
\;+\;
\frac{1}{2n}\,\boldsymbol{\delta}(c)^\top \mathrm{Diag}(\mathbf{h})\,\boldsymbol{\delta}(c).
}
\end{equation}
This quadratic proxy is exact when the loss is quadratic in the activations; otherwise, the approximation error is controlled by the magnitude of third-order derivatives (Appendix~\ref{app:remainder}).

\subsubsection{Optimal Constant and Importance Score}

We now derive closed-form expressions for the optimal intervention constant and the resulting minimal loss increase.

\begin{proposition}[Optimal intervention constant]
\label{prop:optimal-c}
Assume $\mathbf{1}_n^\top \mathbf{h} > 0$ (positive total curvature). The unique minimizer of the quadratic proxy \eqref{eq:taylor-proxy} over $c \in \mathbb{R}$ is:
\begin{equation}
\label{eq:cstar}
c^\star_{\ell,j}
\;=\;
\frac{\mathbf{h}^\top \mathbf{a}_j - \mathbf{1}_n^\top \mathbf{g}}{\mathbf{1}_n^\top \mathbf{h}}
\;=\;
\underbrace{\frac{\sum_{s=1}^n h_s \cdot A^{(\ell)}_{s,j}}{\sum_{s=1}^n h_s}}_{\text{curvature-weighted mean}}
\;-\;
\underbrace{\frac{\sum_{s=1}^n g_s}{\sum_{s=1}^n h_s}}_{\text{gradient correction}}.
\end{equation}
\end{proposition}

\begin{proof}
Define the scaled objective $Q(c) := n \cdot \Delta_{\ell,j}(c)$. Substituting $\boldsymbol{\delta}(c) = c\mathbf{1}_n - \mathbf{a}_j$:
\[
Q(c) \;=\; \mathbf{g}^\top(c\mathbf{1}_n - \mathbf{a}_j) + \frac{1}{2}(c\mathbf{1}_n - \mathbf{a}_j)^\top \mathrm{Diag}(\mathbf{h})(c\mathbf{1}_n - \mathbf{a}_j).
\]
The first-order condition $Q'(c) = 0$ gives:
\[
\mathbf{1}_n^\top \mathbf{g} + c \cdot \mathbf{1}_n^\top \mathbf{h} - \mathbf{h}^\top \mathbf{a}_j = 0
\quad \Longrightarrow \quad
c^\star = \frac{\mathbf{h}^\top \mathbf{a}_j - \mathbf{1}_n^\top \mathbf{g}}{\mathbf{1}_n^\top \mathbf{h}}.
\]
The second derivative $Q''(c) = \mathbf{1}_n^\top \mathbf{h} > 0$ confirms this is a minimum.
\end{proof}

\begin{proposition}[Unit importance score]
\label{prop:score}
The minimized proxy loss increase defines the \emph{importance score}:
\begin{equation}
\label{eq:score}
s_{\ell,j}
\;:=\;
\min_{c \in \mathbb{R}} \Delta_{\ell,j}(c)
\;=\;
\Delta_{\ell,j}(c^\star_{\ell,j}).
\end{equation}
This admits the closed form:
\begin{equation}
\label{eq:score-closed}
\boxed{
s_{\ell,j}
\;=\;
\frac{1}{2n} \mathbf{a}_j^\top \mathrm{Diag}(\mathbf{h}) \mathbf{a}_j
\;-\;
\frac{1}{n} \mathbf{g}^\top \mathbf{a}_j
\;-\;
\frac{(\mathbf{h}^\top \mathbf{a}_j - \mathbf{1}_n^\top \mathbf{g})^2}{2n \cdot \mathbf{1}_n^\top \mathbf{h}}.
}
\end{equation}
\end{proposition}

\begin{proof}
Substitute $c^\star_{\ell,j}$ from \eqref{eq:cstar} into \eqref{eq:taylor-proxy}. Writing $\boldsymbol{\delta}^\star := c^\star \mathbf{1}_n - \mathbf{a}_j$ and using the identity $(c^\star)^2 \mathbf{1}_n^\top \mathbf{h} = (\mathbf{h}^\top \mathbf{a}_j - \mathbf{1}_n^\top \mathbf{g})^2 / (\mathbf{1}_n^\top \mathbf{h})$, algebraic simplification yields \eqref{eq:score-closed}.
\end{proof}

\begin{remark}[Interpretation]
The score $s_{\ell,j}$ quantifies the \emph{irreducible} loss degradation from pruning unit $j$: the best-case impact after optimally choosing the replacement constant. Units with low scores are prime candidates for removal.
\end{remark}

\paragraph{Selection rule.}
For a target sparsity of $k$ units within layer $\ell$, a greedy one-shot strategy selects the $k$ units with smallest scores $\{s_{\ell,j}\}_{j=1}^{d_\ell}$ and applies their corresponding optimal constants $\{c^\star_{\ell,j}\}$.

\subsection{Recovering Known Scores}
\label{app:special}

We now show that several established pruning heuristics emerge as special cases of our interventional framework under specific scoring assumptions. These assumptions are sub-cases of the unified assumption set in \Cref{sec:cmr-theorem}; they are used to rank candidate replacements, not to justify the exact compilation step.

\subsubsection{Mean Replacement Pruning}

The simplest structured pruning heuristic replaces each pruned unit's activation with its empirical mean $\bar{a}_j := \frac{1}{n}\mathbf{1}_n^\top \mathbf{a}_j$. We show this is optimal under natural stationarity conditions.

\begin{lemma}[Optimality of mean replacement]
\label{lem:mean}
Suppose the following conditions hold for unit $(\ell, j)$:
\begin{enumerate}[label=(\roman*), leftmargin=2em, itemsep=1pt]
    \item \textbf{Samplewise gradient stationarity:} $g_s = 0$ for all $s \in [n]$ (every per-sample gradient vanishes at the observed activations);
    \item \textbf{Uniform curvature:} $h_s = \alpha$ for all $s \in [n]$, for some $\alpha > 0$.
\end{enumerate}
Then $c^\star_{\ell,j} = \bar{a}_j$, and the importance score simplifies to:
\begin{equation}
\label{eq:mean-var}
s_{\ell,j}
\;=\;
\frac{\alpha}{2} \cdot \Var[\mathbf{a}_j],
\qquad\text{where}\quad
\Var[\mathbf{a}_j] := \frac{1}{n}\sum_{s=1}^n \bigl(A^{(\ell)}_{s,j} - \bar{a}_j\bigr)^2.
\end{equation}
\end{lemma}

\begin{proof}
Under conditions (i)--(ii), equation \eqref{eq:cstar} reduces to:
\[
c^\star = \frac{\alpha \cdot \mathbf{1}_n^\top \mathbf{a}_j - 0}{\alpha \cdot n} = \frac{1}{n}\mathbf{1}_n^\top \mathbf{a}_j = \bar{a}_j.
\]
With $\mathbf{g} = \mathbf{0}$ the linear term of the proxy
vanishes identically, so the proxy \eqref{eq:taylor-proxy} becomes $\Delta_{\ell,j}(c) \approx \frac{\alpha}{2n}\|\boldsymbol{\delta}(c)\|_2^2$, a pure quadratic in $c$. Minimizing over $c$ at $c = \bar{a}_j$ yields $s_{\ell,j} = \frac{\alpha}{2n}\sum_s (A^{(\ell)}_{s,j} - \bar{a}_j)^2 = \frac{\alpha}{2}\Var[\mathbf{a}_j]$.
\end{proof}

\begin{remark}[When do the conditions hold?]
Condition (i) holds exactly at a critical point of the empirical risk with respect to the unit's activations (per-sample activations are independent coordinates, so a critical point forces every $g_s$ to vanish), and approximately for well-trained networks where per-sample gradients are small. Condition (ii) is reasonable for losses with approximately constant curvature (e.g., squared error) or as a first-order approximation when curvature variation is small. Under the weaker mean-stationarity condition $\mathbf{1}_n^\top \mathbf{g} = 0$, the minimizer is still $c^\star = \bar a_j$, but the minimized score acquires the unit-dependent correction $-n^{-1}\mathbf{g}^\top \mathbf{a}_j$, so ranking by variance is recovered only when this correction vanishes (e.g., $\mathbf{g}$ empirically uncorrelated with $\mathbf{a}_j$).
\end{remark}

\subsubsection{Variance-Based Pruning as a Limiting Case}

We now establish a formal equivalence between our framework and variance-based pruning (VBP) \citep{berisha2025vbp}.

\begin{theorem}[Recovery of VBP]
\label{thm:vbp}
Under the conditions of \Cref{lem:mean}, ranking units by the interventional importance score $s_{\ell,j}$ is equivalent to ranking by activation variance $\Var[\mathbf{a}_j]$. Consequently, variance-based pruning with mean replacement is recovered as a special case of the interventional proxy \eqref{eq:taylor-proxy}.
\end{theorem}

\begin{proof}
By \Cref{lem:mean}, $s_{\ell,j} = \frac{\alpha}{2}\Var[\mathbf{a}_j]$ where $\alpha > 0$ is constant across units within a layer. Since ranking is invariant to positive affine transformations, $\text{rank}(s_{\ell,j}) = \text{rank}(\Var[\mathbf{a}_j])$.
\end{proof}

\paragraph{Interpretation.}
VBP prunes units with \emph{low} activation variance: units whose outputs are nearly constant across inputs and thus carry little information. Our causal perspective reveals that this is optimal precisely when the loss landscape has uniform curvature: low-variance units can be replaced by their mean with minimal interventional effect.

\subsubsection{Beyond Uniform Curvature: The General Case}

When curvature varies across samples, the optimal intervention constant departs from the simple mean. Decomposing \eqref{eq:cstar}:
\begin{equation}
\label{eq:cstar-decomp}
c^\star_{\ell,j}
\;=\;
\underbrace{\bar{a}_j^{(h)}}_{\substack{\text{curvature-weighted} \\ \text{mean}}}
\;-\;
\underbrace{\frac{\bar{g}}{\bar{h}}}_{\substack{\text{gradient} \\ \text{correction}}},
\end{equation}
where we define the curvature-weighted statistics:
\[
\bar{a}_j^{(h)} := \frac{\sum_s h_s A^{(\ell)}_{s,j}}{\sum_s h_s}, \qquad
\bar{g} := \frac{1}{n}\sum_s g_s, \qquad
\bar{h} := \frac{1}{n}\sum_s h_s.
\]

This reveals two sources of departure from mean replacement:
\begin{enumerate}[label=(\alph*), leftmargin=2em, itemsep=2pt]
    \item \textbf{Curvature weighting:} Samples with higher curvature $h_s$ (steeper local loss landscape) contribute more to the optimal constant, prioritizing accurate reconstruction on ``sensitive'' inputs.
    \item \textbf{Gradient correction:} Nonzero average gradient $\bar{g}$ shifts the optimal constant away from the weighted mean, exploiting first-order structure to reduce loss.
\end{enumerate}

\begin{remark}[Computational cost]
Computing $c^\star_{\ell,j}$ and $s_{\ell,j}$ for all units in layer $\ell$ requires $O(n \cdot d_\ell)$ operations given precomputed gradients and Hessian diagonals, linear in both sample count and layer width.
\end{remark}

\subsubsection{Explicit Remainder Bounds for ReLU Networks}
\label{sec:relu-remainder}

\Cref{prop:remainder} bounds the Taylor approximation error in terms of the Hessian Lipschitz constant $\rho$. For ReLU networks, we can derive an explicit expression for $\rho$ in terms of architectural parameters.

\begin{assumption}[Bounded ReLU network]
\label{ass:bounded-relu}
Consider an $L$-layer ReLU network with:
\begin{enumerate}[label=(A\arabic*), leftmargin=2.5em, itemsep=2pt]
    \item Weight matrices satisfying $\|W^{(\ell)}\|_2 \leq M_\ell$ and $\|W^{(\ell)}\|_F \leq F_\ell$
    \item Inputs bounded as $\|x\|_2 \leq R$
    \item Loss function $\ell$ with bounded third derivative in the logit vector: the operator (injective) norm of the third-derivative tensor of $\ell$ with respect to the logits is at most $\kappa$ (for cross-entropy with bounded logits, $\kappa = O(1)$)
    \item The intervention segment considered remains inside one ReLU activation-pattern region; segments that cross region boundaries incur an additional Hessian-jump contribution not tracked here.
\end{enumerate}
\end{assumption}

\begin{theorem}[Local Hessian Lipschitz constant for ReLU networks]
\label{thm:rho-relu}
Under \Cref{ass:bounded-relu}, the interventional risk $L_{\ell,j}(c)$ for unit $j$ in layer $\ell$ has Hessian Lipschitz constant bounded by:
\begin{equation}
\label{eq:rho-bound}
\boxed{
\rho_{\ell,j} \;\leq\; \kappa \cdot \left(\prod_{m=\ell+{2}}^{L} M_m\right)^3 \cdot \|W^{(\ell+1)}_{:,j}\|_2^3{,}
}
\end{equation}
with the empty product equal to $1$ when $\ell + 1 = L$.
\end{theorem}

\begin{proof}
The loss as a function of $a^{(\ell)}_j$ composes: (i) the downstream network $f_{\ell+1:L}$ mapping $a^{(\ell)}_j$ to output, and (ii) the loss function $\ell$.

\textit{Step 1: Downstream Jacobian.}
For ReLU networks, the Jacobian of layer $m$ output w.r.t.\ layer $\ell$ activation is:
\[
\frac{\partial a^{(m)}}{\partial a^{(\ell)}_j} = \left(\prod_{m'=\ell+1}^{m} D^{(m')} W^{(m')}\right) e_j
{\;=\; \left(\prod_{m'=\ell+2}^{m} D^{(m')} W^{(m')}\right) D^{(\ell+1)} W^{(\ell+1)}_{:,j}},
\]
where $D^{(m')} = \mathrm{diag}(\mathbf{1}[z^{(m')} > 0])$ is the ReLU gradient (diagonal, entries in $\{0,1\}$). Thus, since the selector $e_j$ extracts column $j$ of $W^{(\ell+1)}$:
\[
\left\|\frac{\partial a^{(m)}}{\partial a^{(\ell)}_j}\right\|_2 \leq {\|W^{(\ell+1)}_{:,j}\|_2 \cdot} \prod_{m'=\ell+{2}}^{m} \|W^{(m')}\|_2 = {\|W^{(\ell+1)}_{:,j}\|_2 \cdot} \prod_{m'=\ell+{2}}^{m} M_{m'}.
\]

\textit{Step 2: Chain rule for third derivative.}

Within a fixed activation-pattern region the downstream map
$a^{(\ell)}_j \mapsto f_\theta$ is affine, so its own second and third
derivatives vanish and the chain rule leaves exactly one term: writing
$J := \partial f_\theta / \partial a^{(\ell)}_j \in \mathbb{R}^q$,
\[
{
\frac{\partial^3 L}{\partial (a^{(\ell)}_j)^3} =
\sum_{p,q',r} \frac{\partial^3 \ell}{\partial z_p \partial z_{q'} \partial z_r} J_p J_{q'} J_r,
\qquad \left|\frac{\partial^3 L}{\partial (a^{(\ell)}_j)^3}\right| \le \kappa \|J\|_2^3.}
\]

\textit{Step 3: Lipschitz constant.}
Within a fixed ReLU activation-pattern region, the downstream map is affine in $a^{(\ell)}_j$, so the Hessian varies only through the smooth loss derivative. On such a region, applying Step 1 with $m = L$:
\[
{
\rho = \sup \left|\frac{\partial^3 L}{\partial (a^{(\ell)}_j)^3}\right| \leq \kappa \cdot \left(\prod_{m=\ell+2}^{L} M_m\right)^3 \cdot \|W^{(\ell+1)}_{:,j}\|_2^3.}
\]
This is \eqref{eq:rho-bound}.
\end{proof}

\begin{corollary}[Layer-wise remainder scaling]
\label{cor:layer-remainder}
For unit $j$ in layer $\ell$ with optimal intervention $c^\star_j$, under \Cref{ass:bounded-relu}, the Taylor remainder satisfies:
\begin{equation}
\label{eq:layer-remainder}
{
|R_3| \;\leq\; \frac{\kappa}{6n} \left(\prod_{m=\ell+2}^{L} M_m\right)^3 \|W^{(\ell+1)}_{:,j}\|_2^3 \cdot \Bigl(n \bigl[(c^\star_j - \bar a_j)^2 + \mathrm{Var}[\mathbf{a}_j]\bigr]\Bigr)^{3/2}.}
\end{equation}
The bracket reduces to $\mathrm{Var}[\mathbf{a}_j]$ exactly when
$c^\star_j = \bar a_j$, which holds for CMR-Logit and under the
samplewise-stationarity, uniform-curvature conditions of \Cref{lem:mean}.
\end{corollary}

\begin{remark}[Depth dependence]
The bound \eqref{eq:rho-bound} grows exponentially with depth $(L - \ell)$ due to the product of spectral norms. This suggests:
\begin{enumerate}[label=(\alph*), leftmargin=2em, itemsep=2pt]
    \item The quadratic proxy is more accurate for units in \emph{later} layers (smaller $L - \ell$).
    \item Networks with spectral normalization ($M_\ell = 1$) have depth-independent bounds.
\end{enumerate}
\end{remark}

\paragraph{Practical computation.}
For a trained network, $\|W^{(\ell)}\|_2$ can be computed via power iteration, and $\kappa$ depends on the loss (e.g., $\kappa = O(1)$ for cross-entropy with bounded logits). This yields a \emph{checkable} bound on proxy accuracy.

\subsection{Multi-Unit Interventions and the Off-Diagonal Diagnostic}
\label{app:multi}

Thus far, we have analyzed single-unit interventions. In practice, compression requires selecting and intervening on \emph{multiple} units simultaneously. We now extend the quadratic proxy to this setting and characterize when optimal selection decomposes into independent per-unit decisions.

\subsubsection{Joint Quadratic Proxy}

Let $S \subseteq [d_\ell]$ denote a subset of units to prune, with intervention constants $\mathbf{c}_S = (c_j)_{j \in S} \in \mathbb{R}^{|S|}$. Define the stacked perturbation vector:
\[
\boldsymbol{\delta}_S(\mathbf{c}_S) \;:=\; \mathrm{vec}\bigl(\{c_j \mathbf{1}_n - \mathbf{a}_j\}_{j \in S}\bigr) \;\in\; \mathbb{R}^{n|S|},
\]
where $\mathrm{vec}(\cdot)$ concatenates the per-unit perturbations. The joint second-order expansion is:
\begin{equation}
\label{eq:multi-proxy}
\Delta_{\ell,S}(\mathbf{c}_S)
\;\approx\;
\frac{1}{n}\, \mathbf{g}_S^\top \boldsymbol{\delta}_S(\mathbf{c}_S)
\;+\;
\frac{1}{2n}\,\boldsymbol{\delta}_S(\mathbf{c}_S)^\top H_S \,\boldsymbol{\delta}_S(\mathbf{c}_S),
\end{equation}
where $\mathbf{g}_S \in \mathbb{R}^{n|S|}$ is the stacked gradient and $H_S \in \mathbb{R}^{n|S| \times n|S|}$ is the Hessian with respect to the stacked activation matrix $A^{(\ell)}_{:,S}$.

\paragraph{The coupling challenge.}
The Hessian $H_S$ generally contains off-diagonal blocks $H_{jk}$ capturing interactions between units $j$ and $k$. These cross-unit couplings arise from shared downstream paths and make joint optimization over $(S, \mathbf{c}_S)$ computationally intractable for large layers.

\subsubsection{Decoupling via Diagonal Curvature}

For scoring, a standard approximation in second-order pruning \citep{DBLP:conf/nips/CunDS89, DBLP:conf/nips/HassibiS92} drops cross-unit curvature interactions:
\begin{equation}
\label{eq:block-diag}
H_S \;\approx\; \mathrm{blockdiag}\bigl(\mathrm{Diag}(\mathbf{h}_j)\bigr)_{j \in S}.
\end{equation}
Under this approximation, each unit's contribution to the loss is locally independent.

\begin{theorem}[Additivity and optimal selection]
\label{thm:additivity}
Assume the scoring Hessian is replaced by the block-diagonal structure \eqref{eq:block-diag}. Then:
\begin{enumerate}[label=(\roman*), leftmargin=2em, itemsep=3pt]
    \item \textbf{Additivity:} The joint proxy decomposes as a sum of single-unit proxies:
    \begin{equation}
    \label{eq:additivity}
    \Delta_{\ell,S}(\mathbf{c}_S) \;\approx\; \sum_{j \in S} \Delta_{\ell,j}(c_j).
    \end{equation}
    
    \item \textbf{Separable optimization:} The optimal constants minimize independently:
    \begin{equation}
    \label{eq:min-add}
    \min_{\mathbf{c}_S \in \mathbb{R}^{|S|}} \Delta_{\ell,S}(\mathbf{c}_S)
    \;=\;
    \sum_{j \in S} \min_{c_j \in \mathbb{R}} \Delta_{\ell,j}(c_j)
    \;=\;
    \sum_{j \in S} s_{\ell,j}.
    \end{equation}
    
    \item \textbf{Greedy optimality:} For a budget of $k$ units, the subset minimizing the proxy is:
    \begin{equation}
    \label{eq:topk}
    S^\star \;=\; \argmin_{S \subseteq [d_\ell],\, |S| = k} \sum_{j \in S} s_{\ell,j} \;=\; \text{bottom-}k\bigl(\{s_{\ell,j}\}_{j=1}^{d_\ell}\bigr).
    \end{equation}
\end{enumerate}
\end{theorem}

\begin{proof}
\textit{(i)} Under \eqref{eq:block-diag}, both the linear term $\mathbf{g}_S^\top \boldsymbol{\delta}_S$ and the quadratic form $\boldsymbol{\delta}_S^\top H_S \boldsymbol{\delta}_S$ decompose across units, yielding \eqref{eq:additivity}.

\textit{(ii)} With no cross-unit terms, minimization over $\mathbf{c}_S$ separates into $|S|$ independent scalar minimizations, each solved by \Cref{prop:optimal-c}.

\textit{(iii)} Since the minimized objective is a sum of independent scalar scores, the optimal size-$k$ subset consists of the $k$ smallest summands.
\end{proof}

\begin{remark}[Computational complexity]
Under diagonal curvature, selecting $k$ units from a layer of width $d_\ell$ requires: (a) $O(n \cdot d_\ell)$ to compute all scores $\{s_{\ell,j}\}$, and (b) $O(d_\ell \log k)$ to extract the bottom-$k$ via a partial sort. This is linear in layer size, tractable even for wide layers.
\end{remark}

\begin{proposition}[Off-diagonal curvature diagnostic]
\label{prop:offdiag-additivity}
Let $B$ be a block of $b$ candidate units, let $H_B \in \mathbb{R}^{b \times b}$ be the curvature matrix used for a block-level quadratic score, and write $H_B = D_B + E_B$ with $D_B := \mathrm{diag}(H_B)$ and $E_B := H_B - D_B$. Define
\[
\rho_{\mathrm{off}}(B) := \frac{\|E_B\|_F}{\|H_B\|_F},
\]
with $\rho_{\mathrm{off}}(B)=0$ when $H_B=0$. For any block perturbation $\delta_B \in \mathbb{R}^b$, the additivity error made by using $D_B$ instead of $H_B$ satisfies
\[
\left|\frac12 \delta_B^\top H_B \delta_B - \frac12 \delta_B^\top D_B \delta_B\right|
\;\leq\;
\frac12\,\rho_{\mathrm{off}}(B)\,\|H_B\|_F\,\|\delta_B\|_2^2.
\]
If additionally $H_B$ is positive semidefinite (as for Gauss--Newton or Fisher curvature) and $H_B \succeq \mu_B I$ on the span of the considered perturbations, with $\mu_B>0$, then
\begin{align*}
\left|\frac12 \delta_B^\top H_B \delta_B - \frac12 \delta_B^\top D_B \delta_B\right|
&\leq
\rho_{\mathrm{off}}(B)\,\kappa_F(B)\,
\left(\frac12\delta_B^\top H_B \delta_B\right),\\
\kappa_F(B)
&:=
\frac{\|H_B\|_F}{\mu_B}
\leq \sqrt{b}\,\frac{\lambda_{\max}(H_B)}{\mu_B}.
\end{align*}
\end{proposition}

\begin{proof}
The two quadratic scores differ only in the off-diagonal term:
\[
\frac12 \delta_B^\top H_B \delta_B - \frac12 \delta_B^\top D_B \delta_B
= \frac12 \delta_B^\top E_B \delta_B .
\]
By Cauchy--Schwarz and $\|E_B\|_2 \le \|E_B\|_F$,
\[
\left|\frac12 \delta_B^\top E_B \delta_B\right|
\le \frac12\|E_B\|_2\|\delta_B\|_2^2
\le \frac12\|E_B\|_F\|\delta_B\|_2^2
= \frac12\rho_{\mathrm{off}}(B)\|H_B\|_F\|\delta_B\|_2^2.
\]
If $H_B \succeq \mu_B I$ on the relevant span, then $\|\delta_B\|_2^2 \le \mu_B^{-1}\delta_B^\top H_B\delta_B$, which gives the relative bound. The final inequality follows from $\|H_B\|_F \le \sqrt{b}\lambda_{\max}(H_B)$ for a positive semidefinite $b \times b$ matrix.
\end{proof}

\paragraph{When is diagonal curvature plausible?}
The approximation \eqref{eq:block-diag} is exact for the quadratic score when the mixed second derivatives between replaced units vanish. For an affine readout with squared logit distortion, this corresponds to orthogonal output-weight columns for the replaced units. More generally, it is an approximation whose quality is measured by \Cref{prop:offdiag-additivity}; weakly correlated unit sensitivities, sometimes encouraged by decorrelating regularizers \citep{cogswell2015reducing}, make the diagnostic small.

\subsection{Third-Order Remainder Bound}
\label{app:remainder}

The quadratic proxy \eqref{eq:taylor-proxy} truncates a Taylor expansion at second order. We now quantify the approximation error under standard smoothness conditions.

\begin{proposition}[Third-order remainder bound]
\label{prop:remainder}

For sample $s$, write $\ell_s(u) := \ell\bigl(f_\theta^{\dop(\ell,j:=u)}(x_s), y_s\bigr)$ for the per-sample loss as a function of the intervened activation. Suppose each $\ell_s$ has $\rho$-Lipschitz second derivative on the segment $[\min(A^{(\ell)}_{s,j}, c), \max(A^{(\ell)}_{s,j}, c)]$:
\[
\bigl| \ell_s''(u) - \ell_s''(v) \bigr| \;\leq\; \rho |u - v|, \qquad \forall u, v \text{ in the segment}, \; \forall s \in [n].
\]
Then the proxy error satisfies:
\begin{equation}
\label{eq:rem-bound}
\boxed{
\Bigl| \Delta_{\ell,j}(c) - \Bigl(\frac{1}{n}\mathbf{g}^\top \boldsymbol{\delta}(c) + \frac{1}{2n}\boldsymbol{\delta}(c)^\top \mathrm{Diag}(\mathbf{h})\boldsymbol{\delta}(c)\Bigr)\Bigr|
\;\leq\;
\frac{\rho}{6n}\,\|\boldsymbol{\delta}(c)\|_2^3.
}
\end{equation}
\end{proposition}

\begin{proof}

Apply the integral form of Taylor's remainder to each sample
around its observed activation $A^{(\ell)}_{s,j}$, with
$\delta_s(c) = c - A^{(\ell)}_{s,j}$:
\[
{
\ell_s(c) = \ell_s(A^{(\ell)}_{s,j}) + g_s\,\delta_s(c) + \tfrac{1}{2} h_s\,\delta_s(c)^2 + R_{3,s},
\qquad |R_{3,s}| \leq \tfrac{\rho}{6}\,|\delta_s(c)|^3,}
\]
by the per-sample Lipschitz condition. Averaging over samples,
the proxy error is $n^{-1}\sum_s |R_{3,s}| \le \frac{\rho}{6n} \sum_s |\delta_s(c)|^3
= \frac{\rho}{6n}\|\boldsymbol{\delta}(c)\|_3^3
\le \frac{\rho}{6n}\|\boldsymbol{\delta}(c)\|_2^3$,
using $\|x\|_3 \le \|x\|_2$. This is \eqref{eq:rem-bound}.
\end{proof}

\begin{corollary}[Perturbation norm at the optimal constant]
\label{cor:tight}

For every constant $c$, the perturbation norm satisfies the identity
$\|\boldsymbol{\delta}(c)\|_2^2 = n\bigl[(c - \bar a_j)^2 + \Var[\mathbf{a}_j]\bigr]$,
minimized at $c = \bar a_j$ with value $n \Var[\mathbf{a}_j]$. At the
optimal intervention $c^\star_{\ell,j}$, therefore,
$\|\boldsymbol{\delta}(c^\star)\|_2 = \sqrt{n[(c^\star_{\ell,j} - \bar a_j)^2 + \Var[\mathbf{a}_j]]}$,
which equals $\sqrt{n \Var[\mathbf{a}_j]}$ exactly when
$c^\star_{\ell,j} = \bar a_j$ (CMR-Logit; samplewise stationarity with
uniform curvature).
Thus for low-variance units (prime pruning candidates) whose optimal
constants stay near the mean, the cubic remainder is small.
\end{corollary}

\paragraph{Practical implications.}
The bound \eqref{eq:rem-bound} suggests that the quadratic proxy is most accurate precisely for the units we wish to prune: those with low activation variance and hence small $\|\boldsymbol{\delta}(c^\star)\|_2$. For high-variance units (which we retain), proxy accuracy matters less since they are not candidates for removal.

\subsection{Proofs for Main Theory and Named Special Cases}
\label{app:cmr-proofs}

\paragraph{Proof of \Cref{thm:margin}.}
\begin{proof}
Fix an input and interchange intervention pair $(x,I)$, and abbreviate
\[
z_L := z_L^I(x), \qquad z_H := z_H^{\omega(I)}(x), \qquad
\Delta z := z_H-z_L, \qquad y := \hat y_L^I(x).
\]
Suppose $m_I(x)>2\epsilon$ and $\|\Delta z\|_\infty \le \epsilon$.
For every $y' \ne y$,
\[
z_{H,y}-z_{H,y'}
= (z_{L,y}-z_{L,y'}) + (\Delta z_y-\Delta z_{y'})
\ge m_I(x)-|\Delta z_y|-|\Delta z_{y'}|
> 2\epsilon-2\epsilon
=0.
\]
Thus $y$ remains the unique top class for the high-level intervened
model, so disagreement is impossible on the event
$\{m_I(x)>2\epsilon\}\cap\{\|\Delta z\|_\infty\le\epsilon\}$.
Equivalently,
\[
\{\hat y_H^{\omega(I)}(x) \ne \hat y_L^I(x)\}
\subseteq
\{m_I(x)\le 2\epsilon\}
\cup
\{\|\Delta z\|_\infty>\epsilon\}.
\]
Taking probabilities over the verification distribution on $(x,I)$ and
applying the union bound gives the stated event inequality.

For the second term, Markov's inequality applied to the nonnegative
random variable $\|\Delta z\|_\infty^2$ yields
\[
\Pr[\|\Delta z\|_\infty>\epsilon]
\le
\frac{\E[\|\Delta z\|_\infty^2]}{\epsilon^2}.
\]
Since $\|\Delta z\|_\infty^2 \le \|\Delta z\|_2^2 \le q\|\Delta z\|_\infty^2$,
the unnormalized squared logit-distortion objective
$D_2:=\E[\|z_H^{\omega(I)}(x)-z_L^I(x)\|_2^2]$ gives
\[
\Pr[\|\Delta z\|_\infty>\epsilon]\le \frac{D_2}{\epsilon^2}.
\]
If CMR-Logit is reported as the per-logit mean
$\bar D_2:=q^{-1}D_2$, the bound is $q\bar D_2/\epsilon^2$.
\end{proof}

\paragraph{IIA consequence of the margin certificate.}
\begin{proof}
The first display is the theorem plus the Markov bound. Optimizing over
$\epsilon$ gives the variational certificate. If $m_I(x)\ge\gamma$
almost surely, then $M(2\epsilon)=0$ for every $\epsilon<\gamma/2$; take
the limit as $\epsilon$ increases to $\gamma/2$.
\end{proof}

\paragraph{Proof of \Cref{prop:bias-fold}.}
\begin{proof}
Decompose the affine consumer by isolating column $j$:
\[
W a^{(\ell)} + b
= W_{:,j}a^{(\ell)}_j + W_{:,\backslash j}a^{(\ell)}_{\backslash j}+b.
\]
Under the intervention $a^{(\ell)}_j=c$, the first term is the constant
$cW_{:,j}$, which absorbs into the bias. This gives
$W_{:,\backslash j}a^{(\ell)}_{\backslash j}+(b+cW_{:,j})$.
\end{proof}

\paragraph{Proof of \Cref{thm:cmr-risk}.}
\begin{proof}
For a calibration sample $s$, freeze the retained activations
$A_{s,K}$ and view the downstream discrepancy as a function of the
coordinates to be replaced:
\[
r_s(u) := \mathcal{D}_s\!\left(F_{\ell\to L}(A_{s,K},u)\right),
\qquad u\in\mathbb{R}^{|S|},
\]
where $F_{\ell\to L}$ denotes the deterministic downstream computation
from layer $\ell$ to logits or to the scalar task loss, depending on the
choice of $d$. For logit distortion,
$\mathcal{D}_s(z)=d(z,F_{\ell\to L}(A_{s,K},A_{s,S}))$; for supervised
task loss, $\mathcal{D}_s(z)=\ell(z,y_s)$. Define
\[
g_{s,S}:=\nabla r_s(A_{s,S}), \qquad
H_{s,S}:=\nabla^2 r_s(A_{s,S}).
\]
For any replacement $\phi\in\Phi$, write
$\delta_s=\phi(A_{s,K})-A_{s,S}$. Taylor's theorem gives
\[
r_s(A_{s,S}+\delta_s)-r_s(A_{s,S})
=g_{s,S}^{\top}\delta_s+\tfrac12\delta_s^\top H_{s,S}\delta_s
+R_s(\delta_s),
\]
where $R_s$ is the third-order remainder. Dropping $R_s$ and averaging
over $s$ gives the displayed quadratic proxy $Q_S(\phi)$. If the
downstream map is affine in $A_S$ and $d$ is a quadratic logit
distortion, the remainder is zero and the expression is exact.

Now let $\mathcal{G}$ be a partition of $S$ into replacement mechanisms,
for example single units or predefined groups. If
$H_{s,S}=\mathrm{blockdiag}(H_{s,G})_{G\in\mathcal{G}}$ and
$\delta_s=(\delta_{s,G})_{G\in\mathcal{G}}$, then
\[
g_{s,S}^{\top}\delta_s+\tfrac12\delta_s^\top H_{s,S}\delta_s
=
\sum_{G\in\mathcal{G}}
\left(g_{s,G}^{\top}\delta_{s,G}
+\tfrac12\delta_{s,G}^\top H_{s,G}\delta_{s,G}\right).
\]
Averaging over samples yields $Q_S(\phi)=\sum_{G\in\mathcal{G}}Q_G(\phi_G)$.
Diagonal curvature is the special case in which every group has size
one. Positive semidefiniteness is only needed to interpret the resulting
quadratic scores as convex scoring objectives and to ensure uniqueness
when the curvature is positive on the replacement subspace.
\end{proof}

\paragraph{Optimal constant.}
\begin{proof}
With $S=\{j\}$, the perturbation is
$\delta_s=c-A_{s,j}^{(\ell)}$, the sample gradient is $g_s$, and the
sample curvature is $h_s$. The proxy $Q_{\{j\}}$ is exactly the scalar
quadratic in \eqref{eq:taylor-proxy}. Its first-order condition and
positive total curvature condition are those of \Cref{prop:optimal-c}.
\end{proof}

\paragraph{VBP recovery.}
\begin{proof}
Under samplewise stationarity and uniform curvature, \eqref{eq:cstar-summary} gives
$c^\star_{\ell,j}=\bar a_j$. Substituting this constant into the proxy
gives $s_{\ell,j}=(\alpha/2)\Var[\mathbf{a}_j]$ by \Cref{lem:mean}.
Since $\alpha/2$ is positive and common across units in the layer, the
ranking agrees with activation variance, which is precisely
\Cref{thm:vbp}.
\end{proof}

\paragraph{CMR-Logit score.}
\begin{proof}
Assume the map from the layer-$\ell$ activations to the logits is
affine in unit $j$ with coefficient column $W_{:,j}$ (exact at the last
layer, or within a fixed ReLU activation-pattern region; elsewhere the
score is the corresponding next-layer surrogate). Then the logit perturbation induced by replacing unit $j$ with $c$ is
\[
\Delta z_s(c)=W_{:,j}(c-A^{(\ell)}_{s,j}).
\]
Thus the squared logit distortion is
\[
\|\Delta z_s(c)\|_2^2
=\|W_{:,j}\|_2^2(c-A^{(\ell)}_{s,j})^2.
\]
At the original activation the first derivative of this discrepancy is
zero, and the curvature with respect to $A^{(\ell)}_{s,j}$ is the
sample-independent scalar $2\|W_{:,j}\|_2^2$. Therefore
\eqref{eq:cstar-summary}
gives $c^\star=\bar a_j$. Averaging the minimized distortion over
samples gives $\|W_{:,j}\|_2^2 n^{-1}\sum_s(A^{(\ell)}_{s,j}-\bar a_j)^2$,
which is the displayed score.
\end{proof}

\paragraph{Affine WLS score and folding.}
\begin{proof}
The affine special case of \Cref{thm:cmr-risk} for a single
replaced unit $j$ with parent set $P \subseteq K$ minimizes, over
$\theta = (\theta_0, \theta_P)$, the curvature-weighted ridge objective
\[
Q_P(\theta) \;=\;
\mathbf{g}^\top(\mathbf{\Phi}_P\theta - \mathbf{a})
+ \tfrac{1}{2}\,(\mathbf{\Phi}_P\theta - \mathbf{a})^\top D\,(\mathbf{\Phi}_P\theta - \mathbf{a})
+ \tfrac{\lambda}{2}\,\|\theta\|_2^2,
\]
where $\mathbf{\Phi}_P := [\mathbf{1}_n, A_{:,P}]$ is the design matrix,
$\mathbf{a} := \mathbf{a}_j$, $D := \mathrm{Diag}(\mathbf{h})$, and
$\lambda \ge 0$ is an optional ridge parameter.
Differentiate $Q_P$ with respect to $\theta$:
\[
\nabla_\theta
=
\mathbf{\Phi}_P^\top\mathbf{g}
+\mathbf{\Phi}_P^\top D(\mathbf{\Phi}_P\theta-\mathbf{a})
+\lambda\theta.
\]
Setting this gradient to zero gives the normal equations. When
$|P|=0$, the design matrix is $\mathbf{1}_n$, and the scalar normal
equation gives the constant in \eqref{eq:cstar-summary}, with the usual ridge
modification when $\lambda>0$. For folding, write the next affine
consumer as $u=W_Ka_K+W_ja_j+b$. Replacing
$a_j$ by $\theta_0+\theta_P^\top a_P$ gives
$u=W_Ka_K+W_j\theta_P^\top a_P+(b+W_j\theta_0)$, so the parent columns
and bias can be updated exactly as in \Cref{prop:bias-fold}.
\end{proof}

\paragraph{Multi-unit additivity.}
\begin{proof}
The block-diagonal part of \Cref{thm:cmr-risk} gives a sum of independent
group objectives. For single-unit blocks, each summand is the
single-unit score $s_{\ell,j}$ from \Cref{prop:score}. Minimizing a sum
of independent scalar scores subject only to the cardinality
constraint $|S|=k$ selects the $k$ smallest scores, as stated in
\Cref{thm:additivity}.
\end{proof}

\paragraph{Proof of \Cref{prop:relu-scaling-invariance}.}
\begin{proof}
Positive homogeneity of ReLU gives
$\mathrm{ReLU}(s u)=s\,\mathrm{ReLU}(u)$ for every $s>0$, so scaling
the incoming affine parameters of unit $j$ by $s$ changes its activation
from $a_j(x)$ to $s a_j(x)$. The contribution of this unit to the next
preactivation is unchanged because
$(s^{-1}W_{\ell+1,:,j})(s a_j(x))=W_{\ell+1,:,j}a_j(x)$. All other
units and parameters are fixed, hence every subsequent activation and
the final function $f_\theta(x)$ are unchanged. The variance and norm
transform as $\Var[s a_j]=s^2\Var[a_j]$ and
$\|s^{-1}W_{\ell+1,:,j}\|_2^2=s^{-2}\|W_{\ell+1,:,j}\|_2^2$, proving
the product invariance. Applying different scalings to different units
leaves all CMR-Logit scores fixed but multiplies VBP scores by arbitrary
positive factors, so any strict ordering of nonzero-variance units can
be realized by a suitable choice of scalings.
\end{proof}

\subsection{Optional Extension: Multi-Head Attention}
\label{app:attention}

This subsection is an optional extension of the mechanism-replacement
formalism, not a load-bearing claim for the main experiments. Modern
transformer architectures rely on multi-head attention (MHA) as a core
computational primitive; the same intervention logic can treat attention
heads as grouped mechanisms for structured pruning.

\subsubsection{Attention as a Structural Causal Model}

\paragraph{Multi-head attention mechanism.}
Consider an MHA layer with $H$ heads operating on input $X \in \mathbb{R}^{T \times d}$ (sequence length $T$, embedding dimension $d$). Each head $h \in [H]$ computes:
\begin{equation}
\label{eq:head-computation}
\text{head}_h(X) \;=\; \mathrm{softmax}\!\left(\frac{X W^Q_h (X W^K_h)^\top}{\sqrt{d_k}}\right) X W^V_h \;\in\; \mathbb{R}^{T \times d_v},
\end{equation}
where $W^Q_h, W^K_h \in \mathbb{R}^{d \times d_k}$ and $W^V_h \in \mathbb{R}^{d \times d_v}$ are the query, key, and value projections for head $h$. The heads are concatenated and projected:
\begin{equation}
\label{eq:mha-output}
\mathrm{MHA}(X) \;=\; \mathrm{Concat}(\text{head}_1, \ldots, \text{head}_H) W^O \;=\; \sum_{h=1}^{H} \text{head}_h(X) W^O_h,
\end{equation}
where $W^O \in \mathbb{R}^{Hd_v \times d}$ and $W^O_h \in \mathbb{R}^{d_v \times d}$ denotes the block of $W^O$ corresponding to head $h$.

\paragraph{SCM structure.}
The MHA layer induces an SCM where:
\begin{itemize}[leftmargin=1.5em, itemsep=2pt]
    \item \textbf{Exogenous:} Input embeddings $X \sim \mathcal{D}_X$
    \item \textbf{Endogenous:} Head outputs $\{\text{head}_h(X)\}_{h=1}^H$ and attention patterns $\{A_h(X)\}_{h=1}^H$
    \item \textbf{Structural equations:} Defined by \eqref{eq:head-computation}--\eqref{eq:mha-output}
\end{itemize}
Heads contribute \emph{additively} to the layer output via \eqref{eq:mha-output}, which keeps the intervention analysis tractable.

\subsubsection{Hard Interventions on Attention Heads}

\begin{definition}[Head intervention]
\label{def:head-intervention}
For head $h$, a constant intervention $\dop(\mathrm{head}_h := C)$ replaces the head's output with a fixed matrix $C \in \mathbb{R}^{T \times d_v}$, severing its dependence on $X$:
\begin{equation}
\label{eq:head-do}
\mathrm{MHA}^{\dop(h := C)}(X) \;=\; C W^O_h + \sum_{h' \neq h} \mathrm{head}_{h'}(X) W^O_{h'}.
\end{equation}
The zero intervention $\dop(\mathrm{head}_h := 0)$ corresponds to head removal.
\end{definition}

\paragraph{Vectorized notation.}
For a calibration set of $n$ sequences, let $\mathbf{H}_h \in \mathbb{R}^{n \times T \times d_v}$ denote the stacked head outputs. Flattening to $\mathbf{h}_h := \mathrm{vec}(\mathbf{H}_h) \in \mathbb{R}^{nTd_v}$, the intervention $\dop(\mathrm{head}_h := C)$ induces perturbation:
\begin{equation}
\label{eq:head-delta}
\boldsymbol{\delta}_h(C) \;:=\; \mathbf{1}_n \otimes \mathrm{vec}(C) - \mathbf{h}_h \;\in\; \mathbb{R}^{nTd_v}.
\end{equation}

\subsubsection{Quadratic Proxy for Head Removal}

\begin{proposition}[Head importance score]
\label{prop:head-score}
Define the per-sample, per-position, per-dimension gradient and curvature:
\begin{equation}
\label{eq:head-grad-hess}
g_{s,t,i} := \frac{\partial L_s}{\partial [\mathrm{head}_h]_{t,i}}, \qquad
\eta_{s,t,i} := \frac{\partial^2 L_s}{\partial [\mathrm{head}_h]_{t,i}^2},
\end{equation}
with vectorized forms $\mathbf{g}_h, \boldsymbol{\eta}_h \in \mathbb{R}^{nTd_v}$. Under the diagonal scoring approximation, the quadratic proxy for intervening on head $h$ with constant $C$ is:
\begin{equation}
\label{eq:head-proxy}
\Delta_h(C) \;\approx\; \frac{1}{n} \mathbf{g}_h^\top \boldsymbol{\delta}_h(C) + \frac{1}{2n} \boldsymbol{\delta}_h(C)^\top \mathrm{Diag}(\boldsymbol{\eta}_h) \boldsymbol{\delta}_h(C).
\end{equation}
\end{proposition}

For head \emph{removal} (pruning), the natural intervention is $C = 0$:

\begin{proposition}[Head removal score]
\label{prop:head-removal}
The interventional risk increase from removing head $h$ (setting $C = 0$) under the quadratic proxy is:
\begin{equation}
\label{eq:head-removal-score}
\boxed{
s_h^{\mathrm{remove}} \;:=\; \Delta_h(0) \;\approx\; -\frac{1}{n}\mathbf{g}_h^\top \mathbf{h}_h + \frac{1}{2n} \mathbf{h}_h^\top \mathrm{Diag}(\boldsymbol{\eta}_h) \mathbf{h}_h.
}
\end{equation}
Under samplewise gradient stationarity ($\mathbf{g}_h = \mathbf{0}$) and uniform curvature ($\eta_{s,t,i} = \alpha$), this simplifies to:
\begin{equation}
\label{eq:head-variance-score}
s_h^{\mathrm{remove}} \;=\; \frac{\alpha}{2} \cdot \frac{1}{n}\|\mathbf{h}_h\|_2^2 \;=\; \frac{\alpha}{2} \cdot \mathbb{E}_{s}\!\left[\|\mathrm{head}_h(X_s)\|_F^2\right].
\end{equation}
\end{proposition}

\begin{proof}
Substituting $C = 0$ into \eqref{eq:head-proxy} with $\boldsymbol{\delta}_h(0) = -\mathbf{h}_h$ yields \eqref{eq:head-removal-score}. Under the stated conditions, the first term $-\frac{1}{n}\mathbf{g}_h^\top \mathbf{h}_h$ vanishes (samplewise gradient stationarity; note $\mathbf{1}^\top \mathbf{g}_h = 0$ alone would not suffice, since the perturbation $-\mathbf{h}_h$ is sample-dependent) and the second term becomes $\frac{\alpha}{2n}\|\mathbf{h}_h\|_2^2$.
\end{proof}

\begin{remark}[Interpretation]
Equation \eqref{eq:head-variance-score} reveals that under idealized conditions, head importance is proportional to the expected squared Frobenius norm of its output: heads producing larger-magnitude outputs are more important. This is consistent with empirical observations that ``dead'' heads (near-zero output) can be safely pruned \citep{voita2019analyzing, michel2019sixteen}.
\end{remark}

\subsubsection{Optimal Constant Intervention for Heads}

Rather than removing a head entirely, we can replace it with an optimal constant matrix.

\begin{proposition}[Optimal head constant]
\label{prop:head-optimal-c}
Assume $\sum_{s=1}^n \eta_{s,t,i} > 0$ for every position--dimension pair $(t,i)$. The optimal constant intervention $C^\star \in \mathbb{R}^{T \times d_v}$ minimizing \eqref{eq:head-proxy} satisfies, elementwise:
\begin{equation}
\label{eq:head-cstar}
C^\star_{t,i} \;=\; \frac{\sum_{s=1}^n \eta_{s,t,i} \cdot [\mathrm{head}_h(X_s)]_{t,i} - \sum_{s=1}^n g_{s,t,i}}{\sum_{s=1}^n \eta_{s,t,i}}.
\end{equation}
Under uniform curvature and per-coordinate gradient stationarity ($\sum_s g_{s,t,i} = 0$), $C^\star_{t,i} = \frac{1}{n}\sum_{s=1}^n [\mathrm{head}_h(X_s)]_{t,i}$, the sample mean at each position.
\end{proposition}

\begin{proof}
The proxy \eqref{eq:head-proxy} is separable across the $Td_v$ output dimensions under the diagonal scoring approximation. Each dimension reduces to the scalar problem of \Cref{prop:optimal-c}.
\end{proof}

\subsubsection{Bias Folding for Attention Heads}

\begin{proposition}[Head removal via output projection folding]
\label{prop:head-fold}
To implement $\dop(\mathrm{head}_h := C)$, modify the MHA layer as follows:
\begin{enumerate}[label=(\roman*), leftmargin=2em, itemsep=2pt]
    \item Remove head $h$'s parameters $(W^Q_h, W^K_h, W^V_h)$ and the corresponding block $W^O_h$.
    \item 
    For a position-constant intervention $C = \mathbf{1}_T c^\top$
    with $c \in \mathbb{R}^{d_v}$, add the constant contribution to the
    downstream bias: $b' := b + (W^O_h)^\top c$. A position-dependent $C$
    contributes the $T \times d$ matrix $C W^O_h$, which is not a bias
    vector; implementing it requires a fixed sequence length $T$ and a
    stored $T \times d$ additive buffer. In particular, the
    position-dependent optimum $C^\star$ of \Cref{prop:head-optimal-c}
    is foldable only after averaging over positions.
\end{enumerate}
For head removal ($C = 0$), simply delete the head's parameters with no bias update.
\end{proposition}

\paragraph{Computational savings.}
Removing $k$ heads from an $H$-head MHA layer reduces:
\begin{itemize}[leftmargin=1.5em, itemsep=2pt]
    \item Query/Key/Value projections: $(H-k)/H$ of original FLOPs
    \item Attention computation: $(H-k)/H$ of original FLOPs  
    \item Output projection: $(H-k)/H$ of original parameters and FLOPs
\end{itemize}
The savings are linear in the number of pruned heads.

\section{Additional Experimental Results}

\subsection{Ablations and diagnostics (full)}
\label{app:ablations}

\paragraph{Affine parent selection.} On CIFAR-10 ConvNet at keep $0.5$,
the mean IIA over parent selectors is $0.264$ for Pearson, $0.265$ for
output-weight-aware Pearson, and $0.266$ for random; random additionally
has lower KL. Selector identity has no measurable effect on interchange
agreement under the corrected affine WLS solve, and the supervised
affine variant is dominated by CMR-Logit at the same budget. This is
consistent with the affine corollary providing extra capacity rather than a
strong inductive bias.

\paragraph{Calibration size and shift.} CMR-Affine accuracy at keep
$128$ improves from $0.864$ to $0.872$ as calibration size grows from
$500$ to $10000$. Under a class-subset calibration shift (calibrate on
classes $0$--$4$ only; evaluate on $1{,}024$ held-out test images
spanning all classes, under $R = 500$ interchange interventions),
CMR-Logit KL is $0.472/0.480$ at $n=500/2000$, against $0.616/0.599$
for VBP; the logit-distortion criterion therefore transfers better
across calibration distributions than the variance criterion.

\paragraph{Off-diagonal curvature diagnostic.} The mean off-diagonal
Frobenius ratio $\|H_B - \mathrm{diag}(H_B)\|_F / \|H_B\|_F = 0.849$ on
random CIFAR-10 ConvNet blocks ($50$ random $32$-unit blocks of the
$256$-dimensional penultimate representation per seed, three seeds,
$n = 2000$ calibration samples); cross-unit interaction is
therefore substantial at the layer scale tested. \Cref{prop:offdiag-additivity}'s
relative error bound is therefore not tight in this regime; iterative
recompute of the CMR-Const scores (a two-step $256 \to 192 \to 128$
schedule), which avoids the additivity assumption by re-scoring after
each prune step, outperforms one-shot scoring at keep $0.5$ (IIA
$0.614$ versus $0.584$, KL $0.597$ versus $0.710$).

\subsection{Powered MNIST baseline with paired tests}
\label{app:mnist-baseline}

In the standard (non-reparameterized) setting, CMR-Logit has a small but
consistent edge over VBP on a two-hidden-layer
$784$--$512$--$512$--$10$ MNIST MLP, pruning the $512$-unit
penultimate layer, evaluated
under the same $R = 2000$ Bernoulli interchange protocol over ten
seeds. In the table, CMR-Logit appears under its predecessor name
Logit-MSE, and cwvar denotes the curvature-weighted
constant-replacement score of \Cref{prop:score}, the predecessor of
CMR-Const. \Cref{tab:mnist-baseline} summarizes mean $\pm$ standard
deviation; paired bootstrap CIs at keep $=384$ give
$\Delta\mathrm{IIA}(\text{CMR-Logit} - \text{VBP}) = +0.0007$
$[-0.0037, +0.0052]$, not significant; at keep $=256$,
$\Delta\mathrm{IIA} = +0.0056$ $[-0.00085, +0.0123]$, also not
significant, but $\Delta\mathrm{KL} = -0.031$ $[-0.042, -0.019]$ is
significant (lower KL is better).

\begin{table}[t]
\centering
\small
\begin{tabular}{@{}llccc@{}}
\toprule
Keep & Method & Test Acc. & IIA ($p{=}0.5$) & KL ($p{=}0.5$) \\
\midrule
\multirow{3}{*}{384} & Logit-MSE & $0.981\pm0.002$ & $0.952\pm0.006$ & $0.035\pm0.007$ \\
 & VBP & $0.981\pm0.002$ & $0.951\pm0.004$ & $0.038\pm0.007$ \\
 & cwvar & $0.981\pm0.002$ & $0.930\pm0.011$ & $0.084\pm0.023$ \\
\midrule
\multirow{3}{*}{256} & Logit-MSE & $0.980\pm0.001$ & $0.782\pm0.006$ & $0.578\pm0.020$ \\
 & VBP & $0.979\pm0.002$ & $0.776\pm0.010$ & $0.609\pm0.026$ \\
 & cwvar & $0.980\pm0.002$ & $0.732\pm0.012$ & $0.861\pm0.087$ \\
\bottomrule
\end{tabular}
\caption{MNIST baseline summary (mean $\pm$ std over seeds 0--9).}
\label{tab:mnist-baseline}
\end{table}

The qualitative interpretation is consistent with the small CMR-Logit edge
reported on the CIFAR-10 ConvNet in \Cref{sec:exp-verification}: the
abstraction lens does not provide a large IIA win in the standard
setting, but is consistently better on logit fidelity. The reparameterization
stress test (\Cref{sec:exp-invariance}) is where the gap becomes large,
not the standard setting.

\subsection{Boolean circuit sanity check}
\label{app:boolean}

We test the discovery procedure on a controlled three-layer Boolean
MLP trained to fit
$y = \mathrm{XOR}(\mathrm{AND}(x_1, x_2), \mathrm{OR}(x_3, x_4))$
from $\mathbf{x} \in \{0, 1\}^8$, where the four irrelevant input
coordinates $x_5, \ldots, x_8$ are independent Bernoulli$(1/2)$. The
trained MLP has two hidden layers of $64$ units each,
with pruning targeting the $64$-unit penultimate layer; the calibration
set is $n = 2000$ inputs drawn from a $4{,}096$-point dataset with an
$80/20$ train/test split (\Cref{tab:boolean-circuit}). Six seeds, $R = 2000$ Bernoulli swaps at $p = 0.5$.

\begin{table}[t]
\centering
\small
\caption{Boolean circuit interchange-fidelity at keep $=32$ and
$16$ (out of 64). CMR-Logit / Logit-MSE matches VBP on IIA and wins
on KL at keep $=32$, and is essentially tied at keep $=16$; both clearly beat random.}
\label{tab:boolean-circuit}
\begin{tabular}{@{}llccc@{}}
\toprule
Keep & Method & Test Acc. & IIA ($p{=}0.5$) & KL ($p{=}0.5$) \\
\midrule
\multirow{4}{*}{32} & Logit-MSE & $1.000\pm0.001$ & $0.927\pm0.020$ & $0.172\pm0.092$ \\
 & VBP & $0.989\pm0.024$ & $0.923\pm0.025$ & $0.234\pm0.149$ \\
 & cwvar & $0.981\pm0.042$ & $0.864\pm0.025$ & $0.527\pm0.251$ \\
 & Random & $0.921\pm0.075$ & $0.794\pm0.058$ & $1.013\pm0.354$ \\
\midrule
\multirow{4}{*}{16} & Logit-MSE & $0.987\pm0.029$ & $0.814\pm0.029$ & $0.806\pm0.202$ \\
 & VBP & $0.951\pm0.050$ & $0.827\pm0.033$ & $0.790\pm0.212$ \\
 & cwvar & $0.928\pm0.067$ & $0.752\pm0.019$ & $1.234\pm0.216$ \\
 & Random & $0.768\pm0.114$ & $0.710\pm0.067$ & $1.509\pm0.559$ \\
\bottomrule
\end{tabular}

\end{table}

CMR-Logit (= Logit-MSE in predecessor notation) and VBP are essentially
tied on this controlled task, both with high IIA at keep $=32$ ($\sim
0.93$). CMR-Logit attains lower KL (0.172 vs 0.234) and higher Test
accuracy (1.000 vs 0.989) at keep $=32$. At more aggressive keep $=16$,
VBP edges out CMR-Logit by 0.013 IIA, while CMR-Logit retains a Test
accuracy advantage. Random pruning loses substantially at both budgets,
confirming that the task is non-trivial.

\subsection{Affine vs.\ constant comparison at aggressive sparsity}
\label{app:affine-deltas}

The affine corollary of \Cref{thm:cmr-risk} adds capacity but is
metric-dependent: it can improve interchange agreement at aggressive
budgets while degrading KL fidelity. \Cref{tab:affine-deltas} reports
paired-test deltas $\Delta(\text{affine} - \text{constant})$ over ten
seeds with $R = 2000$ swaps, on the predecessor MNIST setup at keep
$\in \{64, 128\}$ and parent count $r \in \{4, 16\}$ with two ridge
levels.

\begin{table}[t]
\centering
\small
\caption{Affine $-$ constant deltas at keep $\in \{64, 128\}$ on
MNIST with paired bootstrap 95\% CIs and seed counts. Positive
$\Delta$ means affine is larger; for IIA and Test acc., positive is
better; for KL, positive is worse.}
\label{tab:affine-deltas}
\begin{tabular}{@{}rrlllrr@{}}
\toprule
keep & $r$ & ridge & metric & $\Delta$ (95\% CI) & seeds & $p$ \\
\midrule
64 & 4 & $10^{-4}$ & IIA ($p{=}0.5$) & 0.0490 [0.0338, 0.0653] & 10 & 0.000 \\
64 & 4 & $10^{-4}$ & KL ($p{=}0.5$) & 1.2275 [1.1261, 1.3275] & 10 & 0.000 \\
64 & 4 & $10^{-4}$ & acc & 0.0564 [0.0315, 0.0855] & 10 & 0.004 \\
64 & 4 & $10^{-2}$ & IIA ($p{=}0.5$) & 0.0519 [0.0378, 0.0668] & 10 & 0.000 \\
64 & 4 & $10^{-2}$ & KL ($p{=}0.5$) & 0.8086 [0.7022, 0.9076] & 10 & 0.000 \\
64 & 4 & $10^{-2}$ & acc & 0.0564 [0.0320, 0.0847] & 10 & 0.003 \\
64 & 16 & $10^{-4}$ & IIA ($p{=}0.5$) & 0.0387 [0.0207, 0.0571] & 10 & 0.004 \\
64 & 16 & $10^{-4}$ & KL ($p{=}0.5$) & 1.9373 [1.7361, 2.1020] & 10 & 0.000 \\
64 & 16 & $10^{-4}$ & acc & 0.0588 [0.0334, 0.0895] & 10 & 0.004 \\
64 & 16 & $10^{-2}$ & IIA ($p{=}0.5$) & 0.0506 [0.0359, 0.0664] & 10 & 0.000 \\
64 & 16 & $10^{-2}$ & KL ($p{=}0.5$) & 1.1775 [1.0215, 1.3148] & 10 & 0.000 \\
64 & 16 & $10^{-2}$ & acc & 0.0588 [0.0334, 0.0893] & 10 & 0.004 \\
128 & 4 & $10^{-4}$ & IIA ($p{=}0.5$) & 0.0130 [0.0075, 0.0184] & 10 & 0.002 \\
128 & 4 & $10^{-4}$ & KL ($p{=}0.5$) & 0.6867 [0.6483, 0.7243] & 10 & 0.000 \\
128 & 4 & $10^{-4}$ & acc & 0.0048 [0.0022, 0.0076] & 10 & 0.008 \\
128 & 4 & $10^{-2}$ & IIA ($p{=}0.5$) & 0.0148 [0.0098, 0.0199] & 10 & 0.000 \\
128 & 4 & $10^{-2}$ & KL ($p{=}0.5$) & 0.4629 [0.4019, 0.5224] & 10 & 0.000 \\
128 & 4 & $10^{-2}$ & acc & 0.0051 [0.0025, 0.0079] & 10 & 0.007 \\
128 & 16 & $10^{-4}$ & IIA ($p{=}0.5$) & 0.0112 [0.0036, 0.0189] & 10 & 0.025 \\
128 & 16 & $10^{-4}$ & KL ($p{=}0.5$) & 0.9836 [0.9079, 1.0482] & 10 & 0.000 \\
128 & 16 & $10^{-4}$ & acc & 0.0052 [0.0025, 0.0081] & 10 & 0.008 \\
128 & 16 & $10^{-2}$ & IIA ($p{=}0.5$) & 0.0138 [0.0076, 0.0199] & 10 & 0.002 \\
128 & 16 & $10^{-2}$ & KL ($p{=}0.5$) & 0.6082 [0.5335, 0.6795] & 10 & 0.000 \\
128 & 16 & $10^{-2}$ & acc & 0.0054 [0.0027, 0.0082] & 10 & 0.005 \\
\bottomrule
\end{tabular}

\end{table}

At keep $=64$ with $r = 16$ and ridge $10^{-2}$, the affine variant
gains $\Delta\mathrm{IIA} = +0.0506$ $[+0.0359, +0.0664]$ and
$\Delta\text{Test acc.} = +0.0588$ $[+0.0334, +0.0893]$, but pays
$\Delta\mathrm{KL} = +1.18$ $[+1.02, +1.31]$ (worse). The trade-off is
sharp: aggressive structured pruning benefits from affine capacity on
class-level metrics, but the higher-capacity replacement also amplifies
distributional distortion. This is the empirical version of the metric
dependence flagged in \Cref{sec:limitations}.

\subsection{Wall-clock timing breakdown}
\label{app:timing}

\Cref{tab:timing-cleardata} decomposes the discovery procedure into score,
compile, and verify stages on the predecessor MNIST setup, averaged
over seeds with 95\% bootstrap CIs.

\begin{table}[t]
\centering
\small
\caption{CMR stage timings on MNIST, mean $\pm$ standard deviation
with 95\% bootstrap CIs. Score is a single forward pass over the
calibration set; compile is the bias / weight folding step; verify
is one $R = 2000$ interchange-intervention pass over precomputed
penultimate features (only the compiled head is re-evaluated per swap).}
\label{tab:timing-cleardata}
\begin{tabular}{@{}llllr@{}}
\toprule
keep & step & variant & mean $\pm$ std & 95\% CI \\
\midrule
128 & compile & affine & $0.3151 \pm 0.0280$\,s & [0.2990, 0.3336]\,s \\
256 & compile & affine & $0.2007 \pm 0.0164$\,s & [0.1915, 0.2120]\,s \\
128 & compile & const & $0.0000841 \pm 0.0000055$\,s & [0.0000807, 0.0000875]\,s \\
256 & compile & const & $0.0000751 \pm 0.0000079$\,s & [0.0000703, 0.0000800]\,s \\
all & score & logit\_mse & $0.1844 \pm 0.0455$\,s & [0.1645, 0.2163]\,s \\
128 & verify & affine & $0.004165 \pm 0.000368$\,s & [0.003944, 0.004399]\,s \\
256 & verify & affine & $0.005842 \pm 0.001113$\,s & [0.005237, 0.006600]\,s \\
128 & verify & const & $0.006269 \pm 0.001389$\,s & [0.005558, 0.007218]\,s \\
256 & verify & const & $0.007750 \pm 0.000831$\,s & [0.007222, 0.008256]\,s \\
\bottomrule
\end{tabular}

\end{table}

Scoring is a one-pass operation independent of replacement class
($\sim 0.18$\,s for $n = 2000$). Constant compilation is essentially
free ($< 0.1$\,ms). Affine compilation requires solving the
curvature-weighted normal equations from the affine special case of
\Cref{thm:cmr-risk} and runs in
0.20--0.32\,s; the result is amortized across all subsequent verify
passes. At this last-layer
budget, scoring and affine compilation dominate total wall-clock;
verification is cheap ($< 8$\,ms for $R = 2000$) because the interchange
swaps re-evaluate only the compiled head on precomputed penultimate
features.

\section{Broader Impacts}
\label{app:broader-impacts}

This work is primarily foundational. A positive impact is that mechanism
replacement can make neural-network reduction more inspectable: reduced
models are compiled artifacts whose behavior can be checked under
interchange interventions, rather than only smaller networks with similar
accuracy. A potential negative impact is overinterpretation: a reduced
computational SCM may be mistaken for a real-world causal explanation or
used as evidence of safety in settings where the validation distribution
is too narrow. We therefore emphasize the locality of the assumptions,
the fixed computational meaning of the state map, and the need for
task-specific validation before deployment.

\end{document}